\documentclass{article}

\usepackage[nonatbib,final]{neurips_2020}

\usepackage[style=numeric-comp,sorting=none,maxbibnames=99,maxcitenames=1]{biblatex}
\AtEveryBibitem{\clearfield{month}}
\AtEveryCitekey{\clearfield{month}}
\AtEveryBibitem{\clearlist{language}}
\renewbibmacro{in:}{}
\usepackage[utf8]{inputenc} % allow utf-8 input
\usepackage[T1]{fontenc}    % use 8-bit T1 fonts

\usepackage[colorlinks,linktoc=all,bookmarks=false]{hyperref}
\usepackage[all]{hypcap}
\hypersetup{citecolor=black}
\hypersetup{linkcolor=black}
\hypersetup{urlcolor=black}

\usepackage{url}            % simple URL typesetting
\usepackage{booktabs}       % professional-quality tables
\usepackage{amsfonts}       % blackboard math symbols
\usepackage{nicefrac}       % compact symbols for 1/2, etc.
\usepackage{microtype}      % microtypography

\usepackage{appendix}

\usepackage[dvipsnames,table]{xcolor}
\usepackage[algoruled]{algorithm2e}
\usepackage{ahw_macros}
\usepackage{parskip}
\usepackage{appendix}
\usepackage{enumerate}

\title{Point process models for sequence detection in high-dimensional neural spike trains}

\author{%
  Alex H. Williams \\
  Department of Statistics \\
  Stanford University \\
  Stanford, CA 94305 \\
  \texttt{ahwillia@stanford.edu} \\
  \And
  Anthony Degleris \\
  Department of Electrical Engineering \\
  Stanford University \\
  Stanford, CA 94305 \\
  \texttt{degleris@stanford.edu} \\
  \AND
  Yixin Wang \\
  Department of Statistics \\
  Columbia University \\
  New York NY 10027\\
  \texttt{yixin.wang@columbia.edu}
  \And 
  Scott W. Linderman \\
  Department of Statistics \\
  Stanford University \\
  Stanford, CA 94305 \\
  \texttt{scott.linderman@stanford.edu} \\
}

\begin{document}

\begin{refsection}

\maketitle

\begin{abstract}
Sparse sequences of neural spikes are posited to underlie aspects of working memory \cite{Goldman2009}, motor production \cite{Hahnloser2002}, and learning \cite{Eichenbaum2014,Mackevicius2019}.
Discovering these sequences in an unsupervised manner is a longstanding problem in statistical neuroscience \cite{Abeles2001, Russo2017, Quaglio2018}.
Promising recent work \cite{Mackevicius2019,Peter2017} utilized a convolutive nonnegative matrix factorization model \cite{Smaragdis2006} to tackle this challenge.
However, this model requires spike times to be discretized, utilizes a sub-optimal least-squares criterion, and does not provide uncertainty estimates for model predictions or estimated parameters.
We address each of these shortcomings by developing a point process model that characterizes fine-scale sequences at the level of individual spikes and represents sequence occurrences as a small number of marked events in continuous time.
This ultra-sparse representation of sequence events opens new possibilities for spike train modeling.
For example, we introduce learnable time warping parameters to model sequences of varying duration, which have been experimentally observed in neural circuits \cite{Davidson2009}.
We demonstrate these advantages on experimental recordings from songbird higher vocal center and rodent hippocampus.
\end{abstract}

\section{Introduction}

Identifying interpretable patterns in multi-electrode recordings is a longstanding and increasingly pressing challenge in neuroscience.
Depending on the brain area and behavioral task, the activity of large neural populations may be low-dimensional as quantified by principal components analysis (PCA) or other latent variable models \cite{Smith2003,Briggman2005,Yu2009,Paninski2010,Macke2011,Pfau2013,Gao2015,Gao2016,Zhao2017,Wu2018,Williams2018,Linderman2019,Duncker2019}.
However, many datasets do not conform to these modeling assumptions and instead exhibit high-dimensional behaviors \cite{Stringer2019}.
\textit{Neural sequences} are an important example of high-dimensional structure: if~$N$ neurons fire sequentially with no overlap, the resulting dynamics are~$N$-dimensional and cannot be efficiently summarized by PCA or other linear dimensionality reduction methods \cite{Mackevicius2019}.
Such sequences underlie current theories of working memory~\cite{Goldman2009,Harvey2012}, motor production \cite{Hahnloser2002}, and memory replay \cite{Davidson2009}. More generally, neural sequences are seen as flexible building blocks for learning and executing complex neural computations \cite{Rabinovich2006,Mackevicius2018,Buzsaki2018}.

\textbf{Prior Work}\quad
In practice, neural sequences are usually identified in a supervised manner by correlating neural firing with salient sensory cues or behavioral actions.
For example, hippocampal place cells fire sequentially when rodents travel down a narrow hallway, and these sequences can be found by averaging spikes times over multiple traversals of the hallway.
After identifying this sequence on a behavioral timescale lasting seconds, a template matching procedure can be used to show that these sequences reoccur on compressed timescales during wake \cite{Pastalkova2008} and sleep \cite{Ji2007}.
In other cases, sequences can be identified relative to salient features of the local field potential (LFP; \cite{Joo2018,Xu2019}).

Developing unsupervised alternatives that directly extract sequences from multineuronal spike trains would broaden the scope of this research and potentially uncover new sequences that are not linked to behaviors or sensations \cite{Tingley2020}.
Several works have shown preliminary progress in this direction
Maboudi et al. \cite{Maboudi2018} proposed fitting a hidden Markov model (HMM) and then identifying sequences from the state transition matrix.
Grossberger et al. \cite{Grossberger2018} and van der Meij \& Voytek \cite{Meij2018} apply clustering to features computed over a sliding window.
Others use statistics such as time-lagged correlations to detect sequences and other spatiotemporal patterns in a bottom-up fashion \cite{Russo2017,Quaglio2018}.

In this paper, we develop a Bayesian point process modeling generalization of \textit{convolutive nonnegative matrix factorization} (convNMF; \cite{Smaragdis2006}), which was recently used by \textcite{Peter2017} and \textcite{Mackevicius2019} to model neural sequences.
Briefly, the convNMF model discretizes each neuron's spike train into a vector of $B$ time bins,~$\bx_n \in \R_+^{_B}$ for neuron~$n$, and approximates this collection of vectors as a sum of convolutions,
$\bx_n \approx \sum_{r=1}^{_R} \bw_{n,r} * \bh_r$. The model parameters ($\bw_{n,r} \in \R_+^{_L}$ and~$\bh_r \in \R^{_B}_+$) are optimized with respect to a least-squares criterion.
Each component of the model, indexed by ${r \in \{1, \hdots, R\}}$, consists of a \textit{neural factor},~$\mathbf{W}_r \in \R^{_{N \times L}}_+$ and a \textit{temporal factor},~$\bh_r$.
The neural factor encodes a spatiotemporal pattern of neural activity over $L$ time bins (which is hoped to be sequence), while the temporal factor indicates the times at which this pattern of neural activity occurs.
A total of $R$ motifs or sequence types, each corresponding to a different component, are extracted by this model.

There are several compelling features of this approach.
Similar to classical nonnegative matrix factorization~\cite{Lee1999}, and in contrast to clustering methods, convNMF captures sequences with overlapping groups of neurons by an intuitive ``parts-based'' representation.
Indeed, convNMF uncovered overlapping sequences in experimental data from songbird higher vocal center (HVC) \cite{Mackevicius2019}.
Further, if the same sequence is repeated with different peak firing rates, convNMF can capture this by varying the magnitude of the entries in~$\bh_r$, unlike many clustering methods.
Finally, convNMF efficiently pools statistical power across large populations of neurons to identify sequences even when the correlations between temporally adjacent neurons are noisy---other methods, such as HMMs and bottom-up agglomerative clustering, require reliable pairwise correlations to string together a full sequence.

\textbf{Our Contributions}\quad
We propose a point process model for neural sequences (PP-Seq) which extends and generalizes convNMF to continuous time and uses a fully probabilistic Bayesian framework.
This enables us to better quantify uncertainty in key parameters---e.g. the overall number of sequences the times at which they occur---and also characterize the data at finer timescales---e.g. whether individual spikes were evoked by a sequence.
Most importantly, by achieving an extremely sparse representation of sequence event times, the PP-Seq model enables a variety of model extensions that are not easily incorported into convNMF or other common methods.
We explore one such extension that introduces time warping factors to model sequences of varying duration, as is often observed in neural data \cite{Davidson2009}.

Though we focus on applications in neuroscience, our approach could be adapted to other temporal point processes, which are a natural framework to describe data that are collected at irregular intervals (e.g. social media posts, consumer behaviors, and medical records) \cite{Moller2003,Gomez2018}.
We draw a novel connection between Neyman-Scott processes \cite{Neyman1958}, which encompass PP-Seq and other temporal point process models as special cases, and mixture of finite mixture models \cite{Miller2018}.
Exploiting this insight, we develop innovative Markov chain Monte Carlo (MCMC) methods for PP-Seq.

\section{Model}
\label{sec:model}

\begin{figure}
\centering
\includegraphics[width=\linewidth]{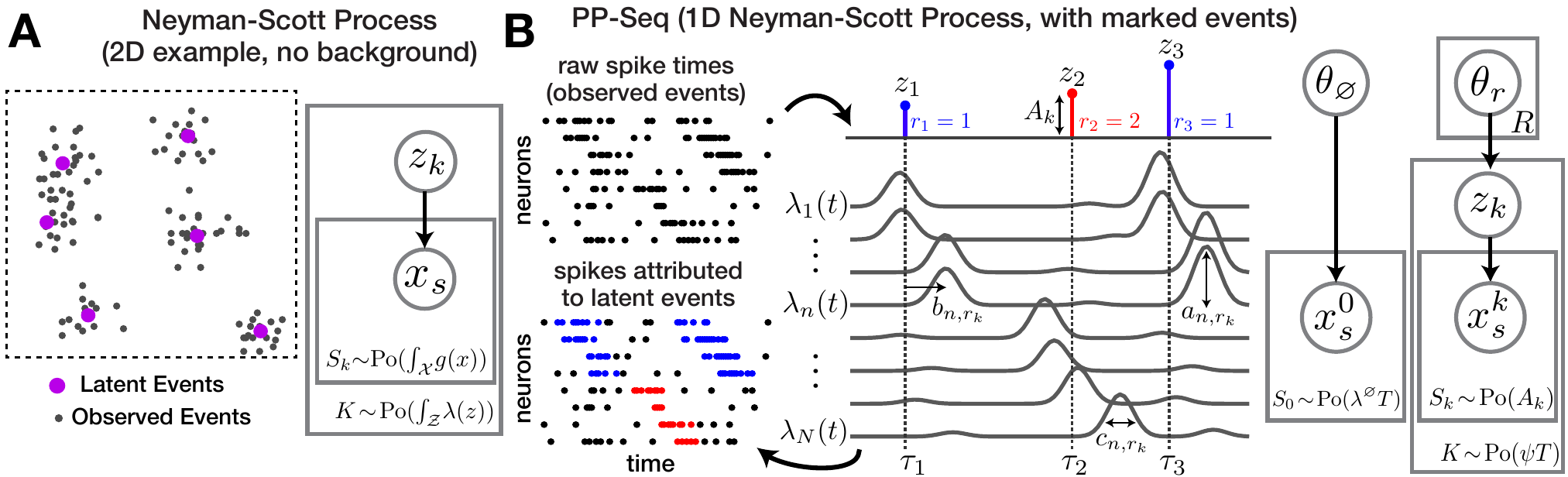}
\caption{
\textit{(A)} Example of a Neyman-Scott process over a 2D region. Latent events (purple dots) are first sampled from a homogeneous Poisson process.
Each latent event then spawns a number of nearby observed events (gray dots), according to an inhomogeneous Poisson process.
\textit{(B)} A spike train can be modeled as a Neyman-Scott process with \textit{marked} events over a 1D interval representing time. Latent events ($z_k$; sequences) evoke observed events ($x_s$; spikes) ordered in a sequence. An example with~$K=3$ latent events evoking~$R=2$ different sequence types (blue \& red) is shown.
}
\label{fig:schematic}
\end{figure}

\subsection{Point Process Models and Neyman-Scott Processes}
\label{sec:neyman-scott-background}

Point processes are probabilistic models that generate discrete sets ${\{x_1, x_2, \hdots \} \triangleq \{x_s\}_{s=1}^S}$ over some continuous space $\mcX$.
Each member of the set $x_s \in \mcX$ is called an ``event.''
A Poisson process~\cite{Kingman2002} is a point process which satisfies three properties.
First, the number of events falling within any region~$V \subset \mcX$ is Poisson distributed.
Second, there exists a function~$\lambda(x): \mathcal{X} \mapsto \R_+$, called the \textit{intensity function}, for which~$\int_V \lambda(x) \, \mathrm{d} x$ equals the expected number of events in~$V$.
Third, the number of events in non-overlapping regions of~$\mcX$ are independent.
A Poisson process is said to be \textit{homogeneous} when~$\lambda(x) = c$ for some constant~$c$.
Finally, a \textit{marked} point process extends the usual definition of a point process to incorporate additional information into each event.
A marked point process on~$\mcX$ generates random tuples~$x_s = (\tilde x_s, m_s)$, where~$\tilde x_s \in \mcX$ represents the random location and~$m_s \in \mcM$ is the additional ``mark'' specifying metadata associated with event~$s$.
See Supplement A for further background details and references.

\textit{Neyman-Scott Processes} \cite{Neyman1958} use Poisson processes as building blocks to model clustered data.
To simulate a Neyman-Scott process, first sample a set of \textit{latent events}~$\{z_k\}_{k=1}^K \subset \mcZ$ from an initial Poisson process with intensity function~$\lambda(z): \mcZ \mapsto \R_+$.
Thus, the number of latent events is a Poisson-distributed random variable,~$K \sim \text{Poisson}(\int_\mcZ \lambda(z) \mathrm{d}z)$.
Given the latent events, the \textit{observed events}~$\{x_s\}_{s=1}^{S}$ are drawn from a second Poisson process with conditional intensity function~${\lambda(x) = \lambda^\varnothing + \sum_{k=1}^{_K} g(x; z_k)}$.
The nonnegative functions~$g(x; z_k)$ can be thought of as impulse responses---each latent event adds to the rate of observed events, and stochastically generates some number of observed offspring events.
Finally, the scalar parameter~$\lambda^\varnothing > 0$ specifies a ``background'' rate; thus, if~$K = 0$, the observations follow a homogeneous Poisson process with intensity~$\lambda^\varnothing$.

A simple example of a Neyman-Scott Process in~$\R^2$ is shown in \cref{fig:schematic}A.
Latent events (purple dots) are first drawn from a homogenous Poisson process and specify the cluster centroids.
Each latent event induces an isotropic Gaussian impulse response.
Observed events (gray dots) are then sampled from a second Poisson process, whose intensity function is found by summing all impulse responses.

Importantly, we \textit{do not} observe the number of latent events nor their locations.
As described below, we use probabilistic inference to characterize this unobserved structure.
A key idea is to attribute each observed event to a latent cause (either one of the latent events or the background process); these attributions are valid due to the additive nature of the latent intensity functions and the \textit{superposition principle} of Poisson processes (see Supplement A).
From this perspective, inferring the set of latent events is similar to inferring the number and location of clusters in a nonparametric mixture model.

\subsection{Point Process Model of Neural Sequences (PP-Seq)}

We model neural sequences as a Neyman-Scott process with \textit{marked} events in a model we call PP-Seq.
Consider a dataset with~$N$ neurons, emitting a total of~$S$ spikes over a time interval~$[0, T]$.
This can be encoded as a set of~$S$ marked spike times---the observed events are tuples~${x_s = (t_s, n_s)}$ specifying the time,~$t_s \in [0, T]$, and neuron,~${n_s \in \{1 \, , \, \hdots \, , \, N\}}$, of each spike.
Sequences correspond to latent events, which are also tuples~$z_k = (\tau_k, r_k, A_k)$ specifying the time,~$\tau_k \in [0, T]$, type,~${r_k \in \{1 \, , \, \hdots \, , \, R\}}$, and amplitude,~$A_k > 0$ of the sequence.
The hyperparameter~$R$ specifies the number of recurring sequence types, which is analogous to the number of components in convNMF.

To draw samples from the PP-Seq model, we first sample sequences (i.e., latent events) from a Poisson process with intensity ${\lambda(z) \triangleq \lambda(\tau, r, A)\!=\!\psi \, \pi_r \, \mathrm{Ga}(A; \alpha, \beta)}$. Here,~$\psi > 0$ sets the rate at which sequences occur within the data,~$\pi \in \Delta_R$ sets the probability of the~$R$ sequence types, and~$\alpha, \beta$ parameterize a gamma density which models the sequence amplitude, $A$.
Note that the number of sequence events is a Poisson-distributed random variable,~$K \sim \text{Poisson}(\psi T)$, where the rate parameter $\psi T$ is found by integrating $\lambda(z)$ over all sequence types, amplitudes, and times.

Conditioned on the set of $K$ sequence events, the firing rate of neuron~$n$ is given by a sum of nonnegative impulse responses:
\begin{equation}
\label{eq:pp-seq-rate-func}
\lambda_n(t) = \lambda^\varnothing_n + \sum_{k=1}^K g_n(t; z_k).
\end{equation}
We assume these impulse responses vary across neurons and follow a Gaussian form:
\begin{equation}
\label{eq:pp-seq-impulse-response}
g_n(t; z_k) = A_k \cdot a_{n r_k} \cdot \mcN(t \mid \tau_k + b_{n r_k}, c_{n  r_k}),
\end{equation}
where~$\mcN(t \mid \mu, \sigma^2)$ denotes a Gaussian density.
The parameters~${a_r = (a_{1r}, \ldots, a_{Nr}) \in \Delta_N}$, ${b_{nr} \in \mathbb{R}}$, and~${c_{nr} \in \mathbb{R}_+}$ correspond to the weight, latency, and width, respectively, of neurons' firing rates in sequences of type~$r$.
Since the firing rate is a sum of non-negative impulse responses, the superposition principle of Poisson processes (see Supplement A) implies that we can view the data as a union of ``background'' spikes and ``induced'' spikes from each sequence, justifying the connection to clustering.
The expected number of spikes induced by sequence~$k$ is:
\begin{equation}
\sum_{n=1}^N \int_0^T g_n(t; z_k) \mathrm{d}t \approx \sum_{n=1}^N \int_{-\infty}^\infty g_n(t; z_k) \mathrm{d}t = A_k,
\end{equation}
and thus we may view~$A_k$ as the amplitude of sequence event $k$.

\Cref{fig:schematic}B schematizes a simple case containing~$K=3$ sequence events and~$R=2$ sequence types.
A complete description of the model's generative process is provided in Supplement B, but it can be summarized by the graphical model in \cref{fig:schematic}B, where we have global parameters~${\Theta = (\theta_\varnothing, \{\theta_r\}_{r=1}^R)}$ with~${\theta_r = (a_r, \{ b_{nr} \}_{n=1}^N, \{ c_{nr} \}_{n=1}^N)}$ for each sequence type, and~$\theta_\varnothing = \{ \lambda^\varnothing_n \}_{n=1}^N$ for the background process.
We place weak priors on each parameter: the neural response weights~$\{a_{nr}\}_{n=1}^N$ follow a Dirichlet prior for each sequence type, and~$(b_{nr}, c_{nr})$ follows a normal-inverse-gamma prior for every neuron and sequence type.
The background rate, $\lambda_n^\varnothing$, follows a gamma prior.
We set the sequence event rate, $\psi$, to be a fixed hyperparameter, though this assumption could be relaxed.

\textbf{Time-warped sequences}\quad
PP-Seq can be extended to model more diverse sequence patterns by using higher-dimensional marks on the latent sequences.
For example, we can model variability in sequence duration by introducing a \textit{time warping factor},~$\omega_k > 0$, to each sequence event and changing \cref{eq:pp-seq-impulse-response} to,
\begin{equation}
\label{eq:warped-impulse-response}
g_n(t; z_k) = A_k \cdot a_{n,r_k} \cdot \mcN(t \mid \tau_k + \omega_k b_{n, r_k}, \omega_k^2 c_{n, r_k}).
\end{equation}
This has the effect of linearly compressing or stretching each sequence in time (when~$\omega_k < 1$ or~$\omega_k > 1$, respectively).
Such time warping is commonly observed in neural data \cite{Duncker2018,Williams2020}, and indeed, hippocampal sequences unfold~$\sim$15-20 times faster during replay than during lived experiences \cite{Davidson2009}.
We characterize this model in Supplement E and demonstrate its utility below.

In principle, it is equally possible to incorporate time warping into discrete time convNMF.
However, since convNMF involves a dense temporal factor matrix~$\bH \in \R_+^{R \times B}$, the most straightforward extension would be to introduce a time warping factor for each component $r \in \{1, \hdots, R\}$ and each time bin $b \in \{1, \hdots, B\}$.
This results in $O(RB)$ new trainable parameters, which poses non-trivial challenges in terms of computational efficiency, overfitting, and human interpretability.
In contrast, PP-Seq represents sequence events as a set of~$K$ latent events in continuous time.
This ultra-sparse representation of sequence events (since~$K \ll R B$) naturally lends itself to modeling additional sequence features since this introduces only~$O(K)$ new parameters.

\section{Collapsed Gibbs Sampling for Neyman-Scott Processes}

Developing efficient algorithms for parameter inference in Neyman-Scott process models is an area of active research \cite{Tanaka2008,Tanaka2014,Kopecky2016,Ogata2019}.
To address this challenge, we developed a collapsed Gibbs sampling routine for Neyman-Scott processes, which encompasses the PP-Seq model as a special case.
The method resembles ``Algorithm 3'' of \textcite{Neal2000}---a well-known approach for sampling from a Dirichlet process mixture model---and the collapsed Gibbs sampling algorithm for ``mixture of finite mixtures'' models developed by \textcite{Miller2018}.
The idea is to partition observed spikes into background spikes and spikes induced by latent sequences, integrating over the sequence times, types, and amplitudes. 
Starting from an initial partition, the sampler iterates over individual spikes and probabilistically re-assigns them to \textit{(a)} the background, \textit{(b)} one of the remaining sequences, or \textit{(c)} to a new sequence.
The number of sequences in the partition, $K^*$, changes as spikes are removed and re-assigned; thus, the algorithm is able to explore the full trans-dimensional space of partitions.

The re-assignment probabilities are determined by the prior distribution of partitions under the Neyman-Scott process and by the likelihood of the induced spikes assigned to each sequence.
We state the conditional probabilities below and provide a full derivation in Supplement D.  
Let~$K^*$ denote the number of sequences in the current partition after spike~$x_s$ has been removed from its current assignment.
(Note that the number of latent sequences~$K$ may exceed~$K^*$ if some sequences produce zero spikes.)
Likewise, let~$u_{s}$ denote the sequence assignment of the~$s$-th spike, where~$u_{s} = 0$ indicates assignment to the background process and~$u_{s} \in \{1, \hdots, K^* \}$ indicates assignment to one of the latent sequence events.
Finally, let~$X_k = \{x_{s'}: u_{s'}=k, s'\neq s\}$ denote the spikes in the~$k$-th cluster, excluding~$x_s$, and let~$S_k = |X_k|$ denote its size.
The conditional probability of the partition under the possible assignments of spike~$x_s$ are,
\begin{align}
p(u_s = 0 \mid x_s, \{X_k\}_{k=1}^{K^*}, \Theta) &\propto (1 + \beta) \, \lambda^\varnothing_{n_s} 
\label{eq:gibbs-bkgd}
\\[.5em]
p(u_s = k \mid x_s, \{X_k\}_{k=1}^{K^*}, \Theta) &\propto (\alpha + S_k ) \, \!\left[\sum_{r_k=1}^R p(r_k \mid X_k) \, a_{n_s r_k} \, 
p(t_s \mid X_k, r_k, n_s) \right]
\label{eq:gibbs-existing}
\\
p(u_s = K^* + 1 \mid x_s, \{X_k\}_{k=1}^{K^*}, \Theta) &\propto \alpha \left ( \frac{\beta}{1 + \beta} \right )^\alpha  \psi \sum_{r=1}^R \pi_r \, a_{n_s r}
\label{eq:gibbs-new}
\end{align}
The sampling algorithm iterates over all spikes~$x_s \in \{x_1, \hdots, x_S\}$ and updates their assignments holding the other spikes' assignments fixed.  
The probability of assigning spike~$x_s$ to an existing cluster marginalizes the time, type, and amplitude of the sequence, resulting in a \textit{collapsed} Gibbs sampler \cite{Liu1994,Neal2000}.   
The exact form of the posterior probability~$p(r_k \mid X_k)$ and the parameters of the posterior predictive~$p(t_s \mid X_k, r_k, n_s)$ in \cref{eq:gibbs-existing} are given in Supplement C.

After attributing each spike to a latent cause, it is straightforward to draw samples over the remaining model parameters---the latent sequences~$\{z_k\}_{k=1}^{K^*}$ and global parameters~$\Theta$.  
Given the spikes and assignments~$\{x_s, u_s\}_{s=1}^S$, we sample the sequences (i.e. their time, amplitude, types, etc.) from the closed-form conditional~$p(z_k \mid \{x_s: u_s=k\}, \Theta)$.  
Given the sequences and spikes, we sample the conditional distribution on global parameters~${p(\Theta \mid \{z_k\}_{k=1}^{K^*}, \{x_s, u_s\}_{s=1}^S})$.  
Under conjugate formulations, these updates are straightforward.
With these steps, the Markov chain targets the posterior distribution on model parameters and partitions.  Complete derivations are in Supplement~D.

\begin{figure}
\centering
\includegraphics[width=\linewidth]{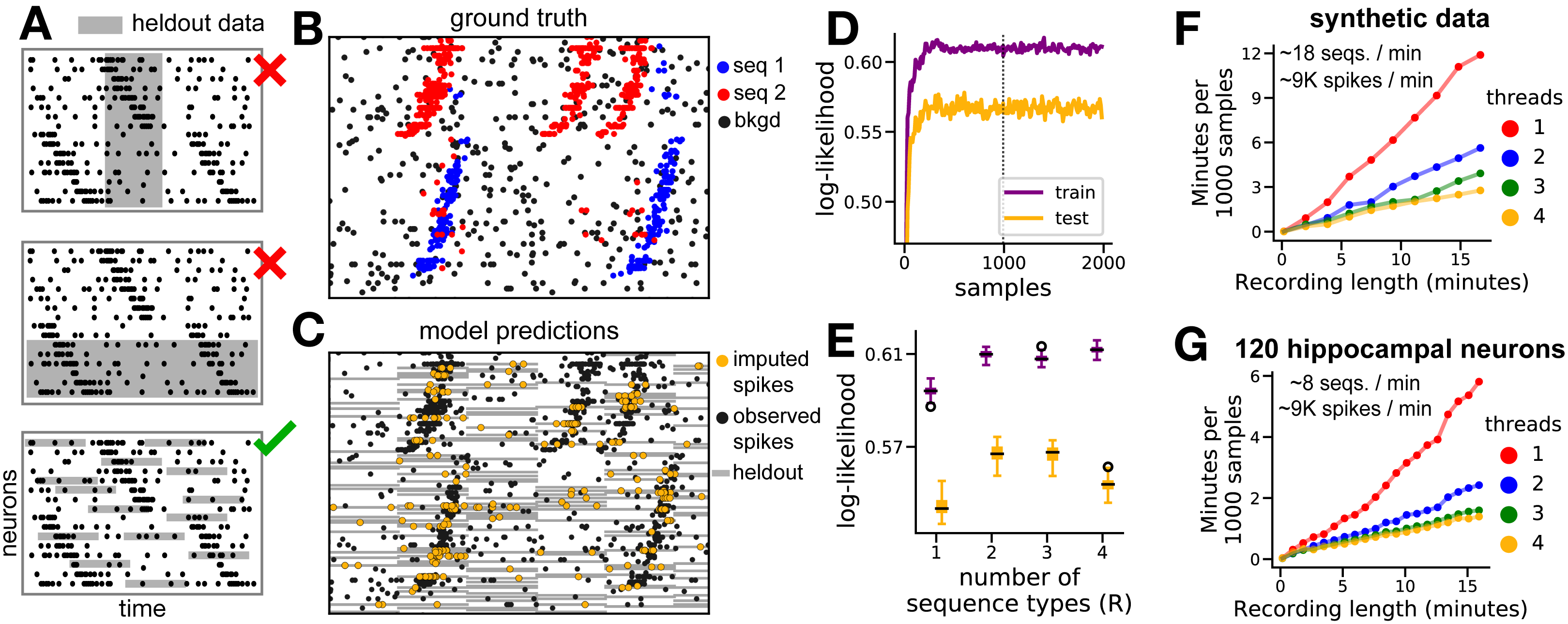}
\caption{
\textit{(A)} Schematic of train/test partitions. We propose a speckled holdout pattern (bottom).
\textit{(B)} A subset of a synthetic spike train containing two sequences types.
\textit{(C)} Same data, but with grey regions showing the censored test set and yellow dots denoting imputed spikes.
\textit{(D)} Log-likelihood over Gibbs samples; positive values denote excess nats per unit time relative to a homogeneous Poisson process baseline.
\textit{(E)} Box plots showing range of log-likelihoods on the train and test sets for different choices of~$R$; cross-validation favors~$R=2$, in agreement with the ground truth shown in panel B.
\textit{(F-G)} Performance benefits of parallel MCMC on synthetic and experimental neural data.
}
\label{fig:crossval}
\end{figure}

\textbf{Improving MCMC mixing times}\quad The intensity of sequence amplitudes~$A_k$ is proportional to the gamma density ~${\mathrm{Ga}(A_k; \alpha, \beta)}$, and these hyperparameters affect the mixing time of the Gibbs sampler.
Intuitively, if there is little probability of low-amplitude sequences, the sampler is unlikely to create new sequences and is therefore slow to explore different partitions of spikes.\footnote{This problem is common to other nonparametric Bayesian mixture models as well \cite[e.g.]{Miller2018}.}
If, on the other hand, the variance of~$\mathrm{Ga}(\alpha, \beta)$ is large relative to the mean, then the probability of forming new clusters is non-negligible and the sampler tends to mix more effectively.
Unfortunately, this latter regime is also probably of lesser scientific interest, since neural sequences are typically large in amplitude---they can involve many thousands of cells, each potentially contributing a small number of spikes \cite{Buzsaki2018,Hahnloser2002}.

To address this issue, we propose an annealing procedure to initialize the Markov chain.
We fix the mean of~$A_k$ and adjust~$\alpha$ and~$\beta$ to slowly lower variance of amplitude distribution.
Initially, the sampler produces many small clusters of spikes, and as we lower the variance of~$\mathrm{Ga}(\alpha, \beta)$ to its target value, the Markov chain typically combines these clusters into larger sequences.
We further improve performance by interspersing ``split-merge'' Metropolis-Hastings updates~\cite{Jain2004,Jain2007} between Gibbs sweeps (see Supplement D.6).
Finally, though we have not found it necessary, one could use convNMF to initialize the MCMC algorithm.

\textbf{Parallel MCMC}\quad
Resampling the sequence assignments is the primary computational bottleneck for the Gibbs sampler.
One pass over the data requires~$O(S K R)$ operations, which quickly becomes costly when the operations are serially executed. 
While this computational cost is manageable for many datasets, we can improve performance substantially by parallelizing the computation~\cite{Angelino2016}.
Given~$P$ processors and a spike train lasting~$T$ seconds, we divide the dataset into intervals lasting~$T / P$ seconds, and allocate one interval per processor.
The current global parameters,~$\Theta$, are first broadcast to all processors.
In parallel, the processors update the sequence assignments for their assigned spikes, and then send back sufficient statistics describing each sequence.
After these sufficient statistics are collected on a single processor, the global parameters are re-sampled and then broadcast back to the processors to initiate another iteration.
This algorithm introduces some error since clusters are not shared across processors.
In essence, this introduces erroneous edge effects if a sequence of spikes is split across two processors.
However, these errors are negligible when the sequence length is much less than~$T / P$, which we expect is the practical regime of interest.

\section{Experiments}

\subsection{Cross-Validation and Demonstration of Computational Efficiency}

\begin{figure}
\centering
\includegraphics[width=\linewidth]{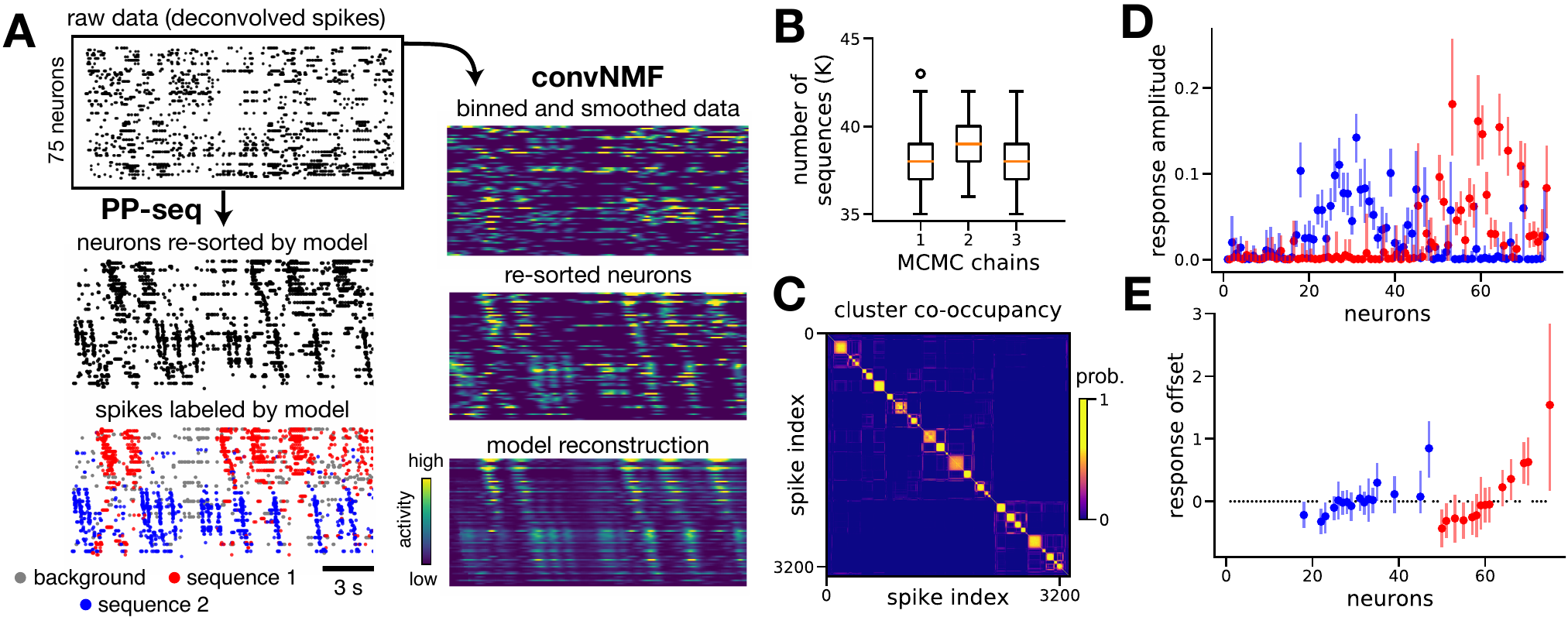}
\caption{
Zebra Finch HVC data.
\textit{(A)} Raw spike train (top) and sequences revealed by PP-Seq (left) and convNMF (right).
\textit{(B)} Box plots summarizing samples from the posterior on number of sequences, $K$, derived from three independent MCMC chains.
\textit{(C)} Co-occupancy matrix summarizing probabilities of spike pairs belonging to the same sequence.
\textit{(D)} Credible intervals for evoked amplitudes for sequence type 1 (red) and 2 (blue).
\textit{(E)} Credible intervals for response offsets (same order and coloring as \textit{D}). Estimates are suppressed for small-amplitude responses (gray dots).
}
\label{fig:bird_fig}
\end{figure}

We evaluate model performance by computing the log-likelihood assigned to held-out data.
Partitioning the data into training and testing sets must be done somewhat carefully---we cannot withhold time intervals completely (as in \cref{fig:crossval}A, \textit{top}) or else the model will not accurately predict latent sequences occurring in these intervals; likewise, we cannot withhold individual neurons completely (as in \cref{fig:crossval}A, \textit{middle}) or else the model will not accurately predict the response parameters of those held out cells.
Thus, we adopt a ``speckled'' holdout strategy \cite{Wold1978} as diagrammed at the bottom of \cref{fig:crossval}A.
We treat held-out spikes as missing data and sample them as part of the MCMC algorithm.  (Their conditional distribution is given by the PP-Seq generative model.)
This approach involving a speckled holdout pattern and multiple imputation of missing data may be viewed as a continuous time extension of the methods proposed by \textcite{Mackevicius2018} for convNMF.

Panels B-E in \cref{fig:crossval} show the results of this cross-validation scheme on a synthetic dataset with~$R=2$ sequence types.
The predictions of the model in held-out test regions closely match the ground truth---missing spikes are reliably imputed when they are part of a sequence (\cref{fig:crossval}C).
Further, the likelihood of the train and test sets improves over the course of MCMC sampling (\cref{fig:crossval}D), and can be used as a metric for model comparison---in agreement with the ground truth, test performance plateaus for models containing greater than~$R=2$ sequence types (\cref{fig:crossval}E).

Finally, to be of practical utility, the algorithm needs to run in a reasonable amount of time.  
\Cref{fig:crossval}G shows that our Julia \cite{Bezanson2017} implementation can fit a recording of 120 hippocamapal neurons with hundreds of thousands of spikes in a matter of minutes, on a 2017 MacBook Pro (3.1 GHz Intel Core i7, 4 cores, 16 GB RAM). 
Run-time grows linearly with the number of spikes, as expected, but even with a single thread it only takes six minutes to perform 1000 Gibbs sweeps on a 15-minute recording with~${\sim\!1.3 \times 10^5}$ spikes.  
With parallel MCMC, this laptop performs the same number of sweeps in under two minutes.
Our open-source implementation is available at:
\begin{center}
\url{https://github.com/lindermanlab/PPSeq.jl}.
\end{center}

\subsection{Zebra Finch Higher Vocal Center (HVC)}
\label{subsec:bird}

We first applied PP-Seq to a recording of HVC premotor neurons in a zebra finch,\footnote{These data are available at \url{http://github.com/FeeLab/seqNMF}; originally published in~\cite{Mackevicius2019}.} which generate sequences that are time-locked to syllables in the bird's courtship song.
\Cref{fig:bird_fig}A qualitatively compares the performance of convNMF and PP-Seq.
The raw data (top panel) shows no visible spike patterns; however, clear sequences are revealed by sorting the neurons lexographically by preferred sequence type and the temporal offset parameter inferred by PP-Seq.
While both models extract similar sequences, PP-Seq provides a finer scale annotation of the final result, providing, for example, attributions at the level of individual spikes to sequences (bottom left of \cref{fig:bird_fig}A).

Further, PP-Seq can quantify uncertainty in key parameters by considering the full sequence of MCMC samples.
\Cref{fig:bird_fig}B summarizes uncertainty in the total number of sequence events, i.e.~$K$, over three independent MCMC chains with different random seeds---all chains converge to similar estimates; the uncertainty is largely due to the rapid sequences (in blue) shown in panel A.
\Cref{fig:bird_fig}C displays a symmetric matrix where element~$(i,j)$ corresponds to the probability that spike~$i$ and spike~$j$ are attributed to same sequence.
Finally, \cref{fig:bird_fig}D-E shows the amplitude and offset for each neuron's sequence-evoked response with 95\% posterior credible intervals.
These results naturally fall out of the probabilistic construction of the PP-Seq model, but have no obvious analogue in convNMF.

\subsection{Time Warping Extension and Robustness to Noise}

\begin{figure}
\centering
\includegraphics[width=\linewidth]{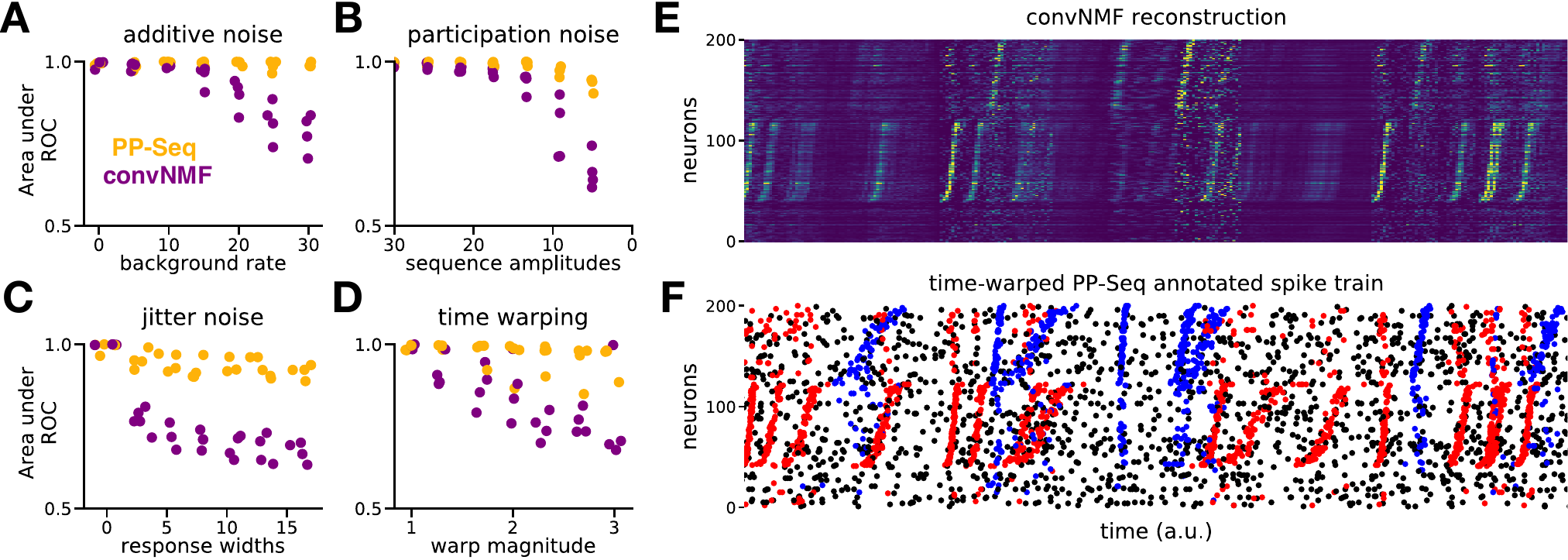}
\caption{
PP-Seq is more robust to various forms of noise than convNMF.
\textit{(A-D)} Comparison of PP-Seq and convNMF to detect sequence times in synthetic data with varying amounts of noise.
Panel D shows the performance of the time-warped variant of PP-Seq (see \cref{eq:warped-impulse-response}).
\textit{(E)} convNMF reconstruction of a spike train containing sequences with 9-fold time warping.
\textit{(F)} Performance of time-warped PP-Seq (see \cref{eq:warped-impulse-response}) on the same data as panel E.
}
\label{fig:noise_tolerance}
\end{figure}

Songbird HVC is a specialized circuit that generates unusually clean and easy-to-detect sequences.
To compare the robustness of PP-Seq and convNMF under more challenging circumstances, we created a simple synthetic dataset with~$R=1$ sequence type and $N=100$ neurons.
We varied four parameters to manipulate the difficulty of sequence extraction: the rate of background spikes, $\lambda^\varnothing_n$ (``additive noise,'' \cref{fig:noise_tolerance}A), the expected value of sequence amplitudes, $A_k$ (``participation noise'', \cref{fig:noise_tolerance}B), the expected variance of the Gaussian impulse responses, $c_{nr}$ (``jitter noise'', \cref{fig:noise_tolerance}C), and, finally, the maximal time warping coefficient (see \cref{eq:warped-impulse-response}; \cref{fig:noise_tolerance}D).
All simulated datasets involved sequences with low spike counts ($\mathbb{E}[A_k]<100$ spikes).
In this regime, the Poisson likelihood criterion used by PP-Seq is better matched to the statistics of the spike train.
Since convNMF optimizes an alternative loss function (squared error instead of Poisson likelihood) we compared the models by their ability to extract the ground truth sequence event times.
Using area under reciever operating characteriztic (ROC) curves as a performance metric (see Supplement F.1), we see favorable results for PP-Seq as noise levels are increased.

We demonstrate the abilities of time-warped PP-Seq further in \cref{fig:noise_tolerance}E-F.
Here, we show a synthetic dataset containing sequences with 9-fold variability in their duration, which is similar to levels observed in some experimental systems \cite{Davidson2009}.
While convNMF fails to reconstruct many of these warped sequences, the PP-Seq model identifies sequences that are closely matched to ground truth.

\subsection{Rodent Hippocampal Sequences}
\label{subsec:hippocampus-results}

Finally, we tested PP-Seq and its time warping variant on a hippocampal recording in a rat making repeated runs down a linear track.\footnote{These data are available at \url{http://crcns.org/data-sets/hc/hc-11}; originally published in \cite{Grosmark2016}.}
This dataset is larger ($T\approx16$ minutes, $S=\,$137,482) and contains less stereotyped sequences than the songbird data.
From prior work \cite{Grosmark2016}, we expect to see two sequences with overlapping populations of neurons, corresponding to the two running directions on the track.
PP-Seq reveals these expected sequences in an unsupervised manner---i.e. without reference to the rat's position---as shown in \cref{fig:hippocampus}A-C.

We performed a large cross-validation sweep over 2,000 random hyperparameter settings for this dataset (see Supplement F.2).
This confirmed that models with $R=2$ sequence performed well in terms of heldout performance (\cref{fig:hippocampus}D).
Interestingly, despite variability in running speeds, this same analysis did not show a consistent benefit to including larger time warping factors into the model (\cref{fig:hippocampus}E).
Higher performing models were characterized by larger sequence amplitudes, i.e. larger values of~$\mathbb{E}[A_k] = \alpha / \beta$, and smaller background rates, i.e. smaller values of~$\lambda^\varnothing$.
Other parameters had less pronounced effects on performance.
Overall, these results demonstrate that PP-Seq can be fruitfully applied to large-scale and ``messy'' neural datasets, that hyperparameters can be tuned by cross-validation, and that the unsupervised learning of neural sequences conforms to existing scientific understanding gained via supervised methods.

\begin{figure}
\centering
\includegraphics[width=\linewidth]{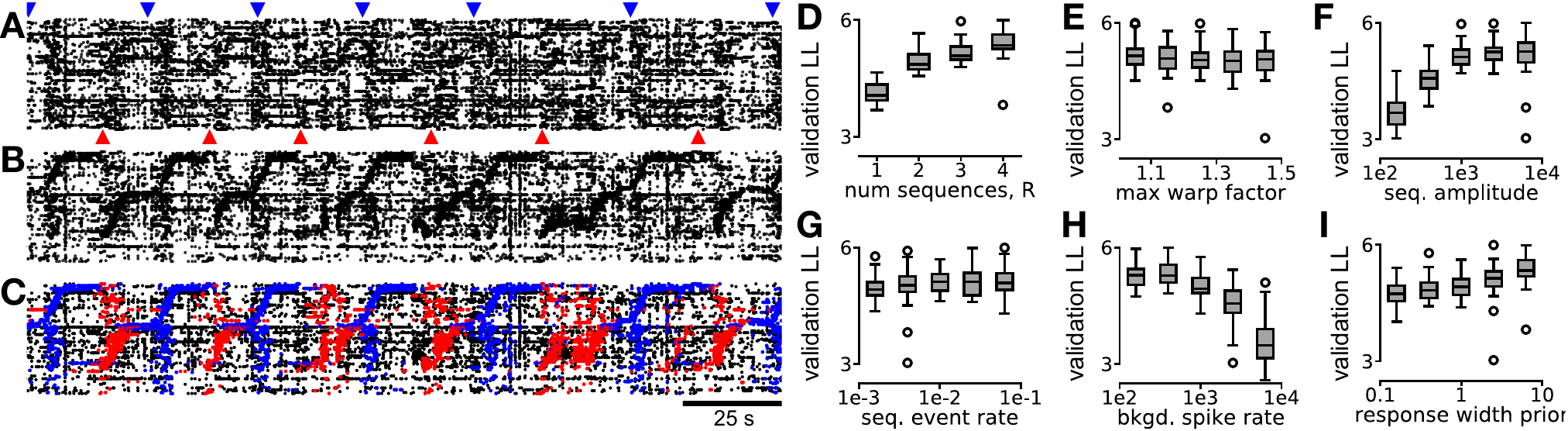}
\caption{
Automatic detection of place coding sequences in rat hippocampus.
\textit{(A)} Raw spike train ($\sim$20\% of the full dataset is shown). Blue and red arrowheads indicate when the mouse reaches the end of the track and reverses direction.
\textit{(B)} Neurons re-sorted by PP-Seq (with $R=2$).
\textit{(C)} Sequences annotated by PP-Seq.
\textit{(D-I)} Validation log-likelihoods as a function of hyperparameter value. Each boxplot summarizes the 50 highest scoring models, randomly sampling the other hyperparameters.
}
\label{fig:hippocampus}
\end{figure}

\section{Conclusion}

We proposed a point process model (PP-Seq) inspired by convolutive NMF \cite{Smaragdis2006,Peter2017,Mackevicius2019} to identify neural sequences.
Both models approximate neural activity as the sum of nonnegative features, drawn from a fixed number of spatiotemporal motif types.
Unlike convNMF, PP-Seq restricts each motif to be a true sequence---the impulse responses are Gaussian and hence unimodal.
Further, PP-Seq is formulated in a probabilistic framework that better quantifies uncertainty (see \cref{fig:bird_fig}) and handles low firing rate regimes (see \cref{fig:noise_tolerance}). 
Finally, PP-Seq produces an extremely sparse representation of sequence events in continuous time, opening the door to a variety of model extensions including the introduction of time warping (see \cref{fig:noise_tolerance}F), as well as other possibilities like truncated sequences and ``clusterless'' observations~\cite{deng2015clusterless}, which could be explored in future work.

Despite these benefits, fitting PP-Seq involves a tackling a challenging trans-dimensional inference problem inherent to Neyman-Scott point processes.
We took several important steps towards overcoming this challenge by connecting these models to a more general class of Bayesian mixture models~\cite{Miller2018}, developing and parallelizing a collapsed Gibbs sampler, and devising an annealed sampling approach to promote fast mixing.
These innovations are sufficient to fit PP-Seq on datasets containing hundreds of thousands of spikes in just a few minutes on a modern laptop.

\subsection*{Acknowledgements}

A.H.W. received funding support from the National Institutes of Health BRAIN initiative (1F32MH122998-01), and the Wu Tsai Stanford Neurosciences Institute Interdisciplinary Scholar Program.
S.W.L. was supported by grants from the Simons Collaboration on the Global Brain (SCGB 697092) and the NIH BRAIN Initiative (U19NS113201 and R01NS113119).
We thank the Stanford Research Computing Center for providing computational resources and support that contributed to these research results.

\subsection*{Broader Impact}

Understanding neural computations in biological systems and ultimately the human brain is a grand and long-term challenge with broad implications for human health and society.
The field of neuroscience is still taking early and incremental steps towards this goal.
Our work develops a general-purpose, unsupervised method for identifying an important structure---neural sequences---which have been observed in a variety of experimental datasets and have been studied extensively by theorists.
This work will serve to advance this growing understanding by providing new analytical tools for neuroscientists. We foresee no immediate impacts, positive or negative, concerning the general public.

\printbibliography
\clearpage
\end{refsection}
\begin{refsection}

\begin{changemargin}{-.5in}{-.5in}

\begin{center}
{\Large Supplementary Material: Point process models for sequence detection in high-dimensional spike trains}
\end{center}
\vspace{2ex}

\begin{appendices}

\section{Background}
\label{sec:background}

This section describes some basic properties and characteristics of Poisson processes that are relevant to understand the PP-Seq model. See \textcite{Kingman2002} for a targeted introduction to Poisson processes and \textcite{Daley2003} for a broader introduction to point processes.

\textbf{Definition of a Poisson Process --- }
A point process model satisfying the following three properties is called a \textit{Poisson process}.
\begin{enumerate}[(i)]
    \item The number of events falling within any region of the sample space $\mcX$ is a random variable following a Poisson distribution. For a region $R \subset \mcX$ we denote the number of events occuring in $R$ as $N(R)$.
    \item There exists a function~$\lambda(x): \mathcal{X} \mapsto \R_+$, called the \textit{intensity function}, which satisfies $N(R) \sim \text{Poisson}(\int_R \lambda(x) dx)$ for every region $R \subset \mcX$.
    \item For any two non-overlapping regions $R$ and $R^\prime$, $N(R)$ and $N(R^\prime)$ are independent random variables.
\end{enumerate}

We write $\{ x_i \}_{i=1}^n \sim PP(\lambda)$ to denote that a set of events $\{ x_i \}_{i=1}^n = \{ x_1, \hdots, x_n \}$ is distributed according to a Poisson process with intensity function $\lambda$.

\textbf{Marked Poisson Processes --- }
A \textit{marked} Poisson process extends the above definition to sets of events marked with additional information. For example, we write $\{ (x_i, m_i) \}_{i=1}^n \sim PP(\lambda(x, m))$ to denote a marked Poisson process where $m_i \in \mcM$ is the ``mark'' associated with each event.
Note that the domain of the intensity function is extended to $\mcX \times \mcM$.
To pick a concrete example of interest, the events could represent spike times, in which case $\mcX$ would be an interval of the real line, e.g. $[0, T]$, representing the recorded period.
Further, $\mcM$ would be a discrete set of neuron labels, e.g. $\{1, \hdots, N\}$ for a recording of $N$ cells.

\textbf{Likelihood function ---}
Given a set of events $\{ x_i \}_{i=1}^n$ in $\mcX$ the likelihood of this set, with respect to Poisson process $PP(\lambda)$ is:
\begin{equation}
\label{eq:poiss-proc-likelihood}
p( \{ x_i \}_{i=1}^n \mid \lambda ) = \prod_{i=1}^n \lambda(x_i) \cdot \exp \left \{ - \int_\mcX \lambda(x) dx \right \}
\end{equation}

\textbf{Sampling ---}
Given an intensity function $\lambda(x)$ and a sample space $\mcX$, we can simulate data from a Poisson process by first sampling the number of events:
\begin{equation}
N \sim \text{Poisson} \left ( \int_\mcX \lambda(x) dx \right )
\end{equation}
and then sampling the location of each event according to the normalized intensity function.
That is, for $i = \{1, \hdots, N\}$, we independently sample the events according to:
\begin{equation}
x_i \sim F \quad \text{where the p.d.f. of $F$ is } \frac{\lambda(x)}{\int_\mcX \lambda(x) dx}
\end{equation}

\textbf{Superposition Principle --- }The union of $K$ independent Poisson processes with nonnegative intensity functions $\lambda_1(x), \hdots, \lambda_K(x)$ is a Poisson process with intensity function $\Lambda(x) = \sum_{k=1}^K \lambda_k(x)$.

As a consequence, we can sample from a complicated Poisson process by first breaking it down into a series of simpler Poisson processes, drawing independent samples from these simpler processes, and taking the union of all events to achieve a sample from the original Poisson process.
We exploit this property in \cref{sec:generative-process} to draw samples from the PP-Seq model.

\section{The PP-Seq Model}

\subsection{Generative Process}
\label{sec:generative-process}

Recall the generative model described in the main text.  Let~$z_k = (\tau_k, r_k, A_k)$ denote the~$k$-th latent sequence, which consists of a time~$\tau_k$, type~$r_k$, and amplitude~$A_k$.  Extensions of the model introduce further properties like time warping factors~$\omega_k$ for each latent sequence.  The latent events (i.e. instantiations of a sequence) are sampled from a Poisson process,
\begin{align}
    \{(\tau_k, r_k, A_k)\}_{k=1}^K &\sim \mathrm{PP}\big(\lambda(\tau, r, A) \big) \\
    \lambda(\tau, r, A) &= \psi \, \pi_r \, \mathrm{Ga}(A; \alpha, \beta).
\end{align}
Note that this is a \textit{homogeneous} Poisson process since the intensity function does not depend on time.
Conditioned on the latent events, the observed spikes are then sampled from a second, inhomogeneous, Poisson process,
\begin{align}
\{(t_s, n_s)\}_{s=1}^S &\sim \mathrm{PP}\left(\lambda(t, n \mid \{(\tau_k, r_k, A_k)\}_{k=1}^K) \right) \\
\lambda(t, n \mid \{(\tau_k, r_k, A_k)\}_{k=1}^K) 
&= \lambda^\varnothing(t, n) + \sum_{k=1}^K g(t, n \mid \tau_k, r_k, A_k) \\
g(t, n \mid \tau_k, r_k, A_k) &= A_k \cdot a_{n r_k} \cdot \mcN(t \mid \tau_k + b_{n r_k}, c_{n  r_k}).
\end{align}
In the main text we abbreviate~$\lambda_n(t) \triangleq \lambda(t, n \mid \{(\tau_k, r_k, A_k)\}_{k=1}^K)$, $\lambda^\varnothing(t, n) \triangleq \lambda^\varnothing_n$, and $g_n(t \mid z_k) \triangleq g(t, n \mid \tau_k, r_k, A_k)$. 

Observe that the firing rates are sums of non-negative intensity functions.  We can use the superposition principle of Poisson processes (see \cref{sec:background}) to write the observed spikes as the union of spikes generated by the background and by each of the latent events.
\begin{align}
    \{(t_s^{(0)}, n_s^{(0)}\}_{s=1}^{S_0} &\sim \mathrm{PP}\big( \lambda^\varnothing(t, n) \big) \\
    \{(t_s^{(k)}, n_s^{(k)}\}_{s=1}^{S_k} &\sim \mathrm{PP}\big( A_k \cdot a_{n r_k} \cdot \mcN(t \mid \tau_k + b_{n r_k}, c_{n  r_k}) \big) & \text{for } &k=1, \ldots, K\\
    \{(t_s, n_s)\}_{s=1}^S &= \bigcup_{k=0}^K \{(t_s^{(k)}, n_s^{(k)}\}_{s=1}^{S_k},
\end{align}
where~$S_k$ denotes the number of spikes generated by the~$k$-th component, with~$k=0$ denoting the background spikes.

Next, note that each of the constituent Poisson processes intensity functions are simple, in that the normalized intensities correspond to simple densities.  In the case of the background, the normalized intensity is categorical in the neuron indices and uniform in time.  The impulse responses, once normalized, are categorical in the neuron indices and conditionally Gaussian in time.  We use these facts to sample the Poisson processes as described in \Cref{sec:background}. 

The final sampling procedure for PP-Seq is as follows.  First sample the global parameters,
\begin{align}
\lambda^\varnothing &\sim \mathrm{Gamma}(\alpha_\varnothing, \beta_\varnothing)
\label{eq:ppseq-lambda0}
\\
\pi^\varnothing &\sim \text{Dirichlet}(\gamma_\varnothing \cdot \boldsymbol{1}_N)
\label{eq:ppseq-pi0}
\\
\pi &\sim \text{Dirichlet}(\gamma \cdot \boldsymbol{1}_R)
\label{eq:ppseq-pi}
\\
\ba_{r} &\sim \text{Dirichlet}(\varphi \cdot \boldsymbol{1}_N)
&& \text{for $r \in \{1, \hdots, R\}$}
\label{eq:ppseq-ar}
\\
c_{nr} &\sim \text{Inv-}\chi^2(\nu, \sigma^2)
&& \text{for $(n,r) \in \{1, \hdots, N\} \times \{1, \hdots, R\}$}
\label{eq:ppseq-cnr}
\\
b_{nr} &\sim \text{Normal}(0, c_{nr} / \kappa)
&& \text{for $(n,r) \in \{1, \hdots, N\} \times \{1, \hdots, R\}$}.
\label{eq:ppseq-bnr}
\end{align}
Then sample the latent events,
\begin{align}
K &\sim \text{Poisson}(\psi \cdot T)
\label{eq:ppseq-K}
\\
r_{k} &\sim \text{Categorical}( \pi )
&& \text{for $k \in \{1, \hdots, K\}$}
\label{eq:ppseq-rk}
\\
\tau_{k} &\sim \text{Uniform}( \, [0, T] \, )
&& \text{for $k \in \{1, \hdots, K\}$}
\label{eq:ppseq-tauk}
\\
A_{k} &\sim \text{Gamma}( \alpha, \beta )
&& \text{for $k \in \{1, \hdots, K\}$}.
\label{eq:ppseq-Ak}
\end{align}
Next sample the background spikes,
\begin{align}
S_{\varnothing} &\sim \text{Poisson}\big(\lambda^\varnothing \cdot T \big)
\label{eq:ppseq-Snot}
\\
n_{0}^{(s)} &\sim \text{Categorical}(\pi^\varnothing)
&& \text{for $s \in \{1, \hdots, S_\varnothing \}$}
\label{eq:ppseq-ns0}
\\
t_{0}^{(s)} &\sim \text{Uniform}( \, [0, T] \, )
&& \text{for $s \in \{1, \hdots, S_\varnothing \}$}
\label{eq:ppseq-ts0}
\end{align}
Note that we have decomposed each neuron's background firing rate as the product: $\lambda^\varnothing_n = \lambda^\varnothing \pi^\varnothing_n$.
Finally, sample the induced spikes in each sequence,
\begin{align}
S_{k} &\sim \text{Poisson}(A_k)
&& \text{for $k \in \{1, \hdots, K\}$}
\label{eq:ppseq-Sk}
\\
n_{k}^{(s)} &\sim \text{Categorical}(\ba_{r_k})
&& \text{for $s \in \{1, \hdots, S_k \}$; for $k \in \{1, \hdots, K \}$}
\label{eq:ppseq-nsk}
\\
t_{k}^{(s)} &\sim \text{Normal}(\tau_k + b_{n_s r_k}, c_{n_s r_k})
&& \text{for $s \in \{1, \hdots, S_k \}$; for $k \in \{1, \hdots, K \}$}
\label{eq:ppseq-tsk}
\end{align}
The hyperparameters of the model are~$\xi = \big \{ \gamma, \gamma_\varnothing, \alpha_\varnothing, \beta_\varnothing, \varphi, \nu, \sigma^2, \kappa, \psi, \alpha, \beta \big \}$, and the global parameters consist of~$\Theta = \{\theta_\varnothing, \{\theta_r\}_{r=1}^R\}$, where~$\theta_\varnothing = \{\lambda^\varnothing, \{ \pi_n^\varnothing \}_{n=1}^N \}$ denotes parameters related to the background spiking process and $\theta_r = \{\pi_r, \mathbf{a}_r, \{b_{nr}, c_{nr}\}_{n=1}^N\}$ denotes parameters associated with each sequence type for $r \in \{1, \hdots, R\}$. 
The resulting dataset is the union of background and induced spikes,~$\bigcup_{k=0}^K \bigcup_{s=1}^{S_k} \big \{ \big (n_k^{(s)}, \, t_k^{(s)} \big ) \big \}$. 

\subsection{Expected Sequence Sizes}
As described in the main text, each sequence type induces a different pattern of activation across the neural population.
In particular, each neuron's response is modeled by a scaled normal distribution described by three variables---the amplitude ($a_{nr}$), the mean ($b_{nr}$), and the variance ($c_{nr}$).

We can interpret $b_{nr}$ as the temporal offset or delay for neuron $n$, while $c_{nr}$ sets the width of the response. 
They are drawn from a \textit{normal-inverse-chi-squared} distribution, which is a standard choice for the conjugate prior when inferring the mean and variance of a univariate normal distribution. Refer to \cite{Gelman2013} (section 3.3) and \cite{Murphy2012} (section 4.6.3.7) for appropriate background material.  

We can interpret $a_{nr}$ as the expected number of spikes emitted by neuron $n$ for a sequence of type $r$ with unit amplitude (i.e., when $A_k = 1$).
By sampling $\ba_r \in \R_+^N$ from a symmetric Dirichlet prior, we introduce the constraint that $\sum_{n} a_{nr} = 1$, which avoids a degeneracy in the amplitude parameters over neurons ($a_{nr}$) and sequences ($A_k$).
As a consequence, we can interpret $A_k$ as the expected number of spikes evoked by latent event $k$ over the full neural population:
\begin{equation}
\mathbb{E}[S_k] = \sum_{n=1}^N \int_0^T A_k \cdot a_{nr} \cdot \mcN(t \mid \tau_k + b_{n r_k}, c_{n r_k}) \, \mathrm{d}t
= A_k \cdot \underbrace{\small{\sum}_{n=1}^N a_{nr}}_{=1} \cdot  \underbrace{\int_0^T \mcN(t \mid \tau_k + b_{n r_k}, c_{n r_k}) \, \mathrm{d}t}_{\approx 1} \approx A_k.
\end{equation}
This approximation ignores boundary effects, but it is justified when $0 \ll \tau_k \ll T$ and $b_{nr}, \sqrt{c_{nr}} \ll T$; i.e. when the sequences are short relative to the interval~$[0, T]$.  This condition will generally be satisfied.

\subsection{Log Likelihood}
\label{sec:ppseq-likelihood}

Given a set of latent events $\{z_k\} = \{z_k\}_{k=1}^K$ and sequence type parameters $\{\theta_r \} = \{\theta_r \}_{r=1}^R$, we can evaluate the likelihood of the dataset $\{x_s\} = \{x_s \}_{s=1}^S$ as:
\begin{equation}
\label{eq:ppseq-log-likelihood}
\log p ( \{ x_s \} \mid \{ z_k \}, \{ \theta_r \} )  = \underbrace{\sum_{s=1}^S \log \left [ \lambda^\varnothing \pi^\varnothing_{n_s} + \sum_{k=1}^K A_k a_{n_s r_k} \mcN(t_s \mid \tau_k + b_{n_s r_k}, c_{n_s r_k}) \right ]}_{(*)} - \underbrace{\lambda^\varnothing T - \sum_{k=1}^K A_k}_{(**)}
\end{equation}
where terms $(*)$ and $(**)$ respectively correspond to the logarithms of $[\, \prod_i \lambda(x_i) \, ]$ and $[\, \exp [ - \int_R \lambda(x) \mathrm{d}x ] \, ]$ appearing in \Cref{eq:poiss-proc-likelihood}.

\section{Complete Gibbs Sampling Algorithm}
\label{sec:gibbs}

In this section we describe the collapsed Gibbs sampling routine for PP-Seq in detail. It adapts ``Algorithm 3'' of \textcite{Neal2000} and the extension by \textcite{Miller2018} to draw approximate samples from the model posterior.
To do this, we introduce auxiliary parent variables $u_s \in \{0\} \cup \mathbb{Z}_+$ for every spike.
These variables denote assignments of each spike to of the latent sequence events, $u_s = k$ for some $k \in \{0 \hdots K\}$, or to the background process, $u_s = 0$.
Let~$X_k = \{x_s: u_s=k\}$ denote the set of spikes assigned to sequence~$k \in \{1, \ldots, K\}$ under parent assignments~$\{u_s\}_{s=1}^S$.
Similar assignment variables are used in standard Gibbs sampling routines for finite mixture models (see, e.g., section 24.2.4 in \cite{Murphy2012}) and for Dirichlet Process Mixture (DPM) models (see \cite{Neal2000}; section 25.2.4 in \cite{Murphy2012}).
In these contexts, each $u_s$ may be thought of as a sequence assignment or indicator variable.
In the present context of the PP-Seq model, the attributions of spikes to latent events follows from the superposition principle of Poisson processes (see \cref{sec:background}), which allows us to decompose each neuron's firing rate function into a sum of simple functions (in this case, univariate Gaussians) and thus derive efficient Gibbs sampler updates.

\begin{algorithm}

\vspace{.5em}

\KwIn{Spikes $\{x_1, \ldots, x_S\}$ and hyperparameters $\xi$.}

\vspace{0.5em}

Initialize by assigning all spikes to the background: $u_s = 0$ for $s \in \{1 \hdots S\}$.

\vspace{0.5em}

\SetKwBlock{Loop}{repeat}{end}

\Loop($M$ times to draw $M$ samples){
    
  1. Resample parent assignments, integrating over latent events:
  
  \vspace{.6em}
  
  \hspace{1.5em} \textbf{for} $s = 1,\, \hdots \,, S$. Remove $x_s$ its current assignment and place it...
  
  \vspace{.8em}
  
  \hspace{2.5em} a. in the background, $u_s = 0$, with probability $\propto (1 + \beta) \, \lambda^\varnothing_{n_s}$

  \vspace{.9em}

  \hspace{2.5em} b. in cluster $u_s = k$, with probability $\propto (\alpha + |X_k| ) \, \left[\sum_{r_k=1}^R p(r_k \mid X_k) \, a_{n_s r_k} \, 
p(t_s \mid X_k, r_k, n_s) \right]$

  \vspace{.8em}
  
  \hspace{2.5em} c. in a new cluster, $u_s = K + 1$, with probability $\propto \alpha \left ( \frac{\beta}{1 + \beta} \right )^\alpha  \psi \sum_{r=1}^R \pi_r \, a_{n_s r}$
  
  \vspace{.2em}

  \hspace{1.5em} \textbf{end}
  
  \vspace{.6em}

  2. Resample latent events 
  \vspace{.6em}
  
  \hspace{1.5em} \textbf{for} $k = 1,\, \hdots \,, K$ sample $r_k, \tau_k, A_k \sim p(r_k, \tau_k, A_k \mid X_k, \Theta)$ with the following steps:
  
  \vspace{.8em}
  
  \hspace{2.5em} a. Sample~$r_k \sim p(r_k \mid X_k, \Theta, \xi)$ where
  \begin{align*}
      p(r_k = r \mid X_k, \Theta, \xi) &\propto 
      \pi_{r_k} \left( \prod_{x_s \in X_k} a_{n_s r_k} \right) \frac{Z({\textstyle \sum_{x_s \in X_k}} J_{sk}, {\textstyle \sum_{x_s \in X_k}} h_{sk})}{\prod_{x_s \in X_k}Z(J_{sk}, h_{sk})}
  \end{align*}
  \hspace{3em} and where $J_{sk} = 1/c_{n_s r_k}$,~$h_{sk} = (t_s - b_{n_s r_k})/c_{n_s r_k}$, and $Z(J, h) = (2 \pi)^{1/2} J^{-1/2} \exp\left \{\tfrac{1}{2} h^2 J^{-1} \right\}$

  \vspace{.6em}

  \hspace{2.5em} b. Sample~$\tau_k \sim p(\tau_k \mid r_k, X_k, \Theta)$ where,
  \begin{align*}
      p(\tau_k \mid r_k, X_k, \Theta) &= \mcN(\tau_k \mid \mu_k, \sigma_k^2)
  \end{align*}
  \hspace{3em} and where $\sigma_k^2 = (\sum_{x_s \in X_k} J_{sk})^{-1}$ and $\mu_k = \sigma_k^2 (\sum_{x_s \in X_k} h_{sk})$

  \vspace{.6em}
  
  \hspace{2.5em} c. Sample~$A_k \sim \mathrm{Ga}(\alpha + |X_k|, \beta + 1)$.

  \vspace{.2em}

  \hspace{1.5em} \textbf{end}
  
  \vspace{.6em}
  
  3. Resample global parameters

  \vspace{.6em}
  
  \hspace{1.5em} a. Sample~$\lambda_n^\varnothing \sim \mathrm{Ga}(\alpha_\varnothing + \sum_{x_s \in X_0} \mathbb{I}[n_s = n], \, \beta_\varnothing + T)$ for $n = 1, \ldots, N$.
  
  \vspace{.6em}
  
  \hspace{1.5em} b. Sample $\pi \sim \mathrm{Dir}\left(\left[\gamma + \sum_k \mathbb{I}[r_k = 1], \ldots, \gamma + \sum_k \mathbb{I}[r_k = R] \right] \right)$
  
  \vspace{.6em}
  
  \hspace{1.5em} c. Sample $\mathbf{a}_r \sim \mathrm{Dir}\left( \left[
  \varphi + \sum_k \sum_{x_s \in X_k} \mathbb{I}[r_k = r \wedge n_s=1], \ldots, 
  \varphi + \sum_k \sum_{x_s \in X_k} \mathbb{I}[r_k = r \wedge n_s=N] \right] \right)$
  
  \vspace{.6em}
  \hspace{1.5em} d. Sample $b_{nr}, c_{nr} \sim \mathrm{Inv-}\chi^2(c_{nr}; \nu_{nr}, \sigma_{nr}^2) \, \mcN(b_{nr} \mid \mu_{nr}, c_{nr} / \kappa_{nr})$ for $n=1\ldots,N$, $r=1\ldots,R$
    where
    \begin{align*}
        S_{nr} &= \sum_{k=1}^{K^*} \sum_{x_s \in X_k} {\mathbb{I}[r_k=r \wedge n_s=n]} \\
        \nu_{nr} &= \nu + S_{nr} & 
        \sigma_{nr}^2 &= \frac{\nu \sigma^2 + \sum_{k=1}^{K^*} \sum_{x_s \in X_k} (t_s - \tau_k)^2 \cdot \mathbb{I}[r_k=r \wedge n_s=n]}{\nu + S_{nr}}, \\
        \kappa_{nr} &= \kappa + S_{nr} &
        \mu_{nr} &= \frac{\sum_{k=1}^{K^*} \sum_{x_s \in X_k} (t_s - \tau_k) \cdot \mathbb{I}[r_k=r \wedge n_s=n]}{\kappa + S_{nr}}.
    \end{align*}
    
  \vspace{.2em}
}
\caption{Collapsed Gibbs sampling for PPSeq model}
\label{alg:collapsedgibbs}
\end{algorithm}

The Gibbs sampling routine is summarized in Algorithm \ref{alg:collapsedgibbs}.
While the approach closely mirrors existing work by \textcite{Neal2000,Miller2018} there are a couple notable differences.
First, the background spikes in the PP-Seq model do not have an obvious counterpart in these prior algorithms, and this feature of the model renders the partition of datapoints only \textit{partially} exchangeable (intuitively, the set of datapoints defined by $u_s = 0$ is treated differently from all other groups).
Second, the factors $S_k + \alpha$ and $\alpha \, \left ( \frac{\beta}{1 + \beta} \right )^\alpha \psi$ in the collapsed Gibbs updates are specific to the special form of the PP-Seq model as a Neyman-Scott process; for example, in the analogous Gibbs sampling routine for DPMs, $(\alpha + |X_k|)$ is replaced with $|X_k|$, and $\alpha \, \psi \, T \left ( \frac{\beta}{1 + \beta} \right )^\alpha$ is replaced with the concentration parameter of the Dirichlet process.

\section{Derivations}

\subsection{Prior Distribution Over Partitions}
\label{subsec:partitions}

First we derive the prior distribution on the partition of spikes under a Neyman-Scott process.
Technically, we derive the prior distribution on spike \emph{indices},~$\{1, \ldots, S\}$, not yet including the spike times and neurons.  
A partition of indices is a set of disjoint, non-empty sets whose union is~$\{1, \ldots, S\}$.  
We represent the partition as~$\mcI = \{\mcI_k\}_{k=0}^{K^*}$, where~$\mcI_k \subset \{1, \ldots, S\}$ contains the indices of spikes assigned to sequence~$k$, with~$k=0$ denoting the background spikes.  
Here,~$K^*$ denotes the number of non-empty sequences in the partition, and~$|\mcI_k|$ denotes the number of spikes assigned to the~$k$-th sequence.  
Likewise,~$|\mcI_0|$ denotes the number of spikes assigned to the background.

\begin{proposition}
The prior probability of the partition, integrating over the latent event parameters, is,
\begin{align}
\label{eq:pepdist}
p(\mcI \mid \Theta, \xi) = V(K^*; \xi) \, \frac{\mathrm{Po}(|\mcI_0| ; \lambda^\varnothing T) \, |\mcI_0|!}{S! \, (1+\beta)^{S-|\mcI_0|}} 
\prod_{k=1}^{K^*} \frac{\Gamma(|\mcI_k|+\alpha)}{\Gamma(\alpha)},
\end{align}
where $\alpha$ and $\beta$ are the hyperparameters of the gamma intensity on amplitudes, ${ \lambda^\varnothing T}$ is the expected number of background events, and 
\begin{align}
\label{eq:Vncoef}
V(K^*; \xi) = \sum_{K=K^*}^\infty \mathrm{Po}(K; \psi T) \, \frac{K!}{(K-K^*)!} \, \left(\frac{\beta}{1+\beta}\right)^{\alpha K}.
\end{align}
\end{proposition}

\begin{proof}
Our derivation roughly follows that of~\textcite{Miller2018} for mixture of finite mixture models, but we adapt it to Neyman-Scott processes.
The major difference is that the distribution above is over partitions of random size.
First, return to the generative process described in \Cref{sec:generative-process} and integrate over the latent event amplitudes to obtain the marginal distribution on sequence sizes given the total number of sequences~$K$.  We have,
\begin{align}
    p(S_0, \hdots, S_K \mid K, \Theta, \xi) &= \mathrm{Po}(S_0; \lambda^\varnothing T) \prod_{k=1}^K \int \mathrm{Po}(S_k; \gamma_k) \, \mathrm{Ga}(A_k; \alpha, \beta) \mathrm{d} A_k \\
    &= \mathrm{Po}(S_0; \lambda^\varnothing T) \prod_{k=1}^K \mathrm{NB}(S_k; \alpha, (1+\beta)^{-1}) \\
    &= \frac{1}{S_0!} (\lambda^\varnothing T)^{N_0} e^{-\lambda^\varnothing T} \prod_{k=1}^K \frac{\Gamma(S_k + \alpha)}{N_k! \, \Gamma(\alpha)} \left( \frac{\beta}{1+\beta} \right)^\alpha \left(\frac{1}{1+\beta} \right)^{S_k} \\
    &= \frac{1}{S_0!} (\lambda^\varnothing T)^{S_0} e^{-\lambda^\varnothing T} \left( \frac{\beta}{1+\beta} \right)^{K\alpha} \left(\frac{1}{1+\beta} \right)^{S-S_0} \prod_{k=1}^K \frac{\Gamma(S_k + \alpha)}{S_k! \, \Gamma(\alpha)} .
\end{align}
Let $\{u_s\}_{s=1}^S$ denote a set of sequence assignments.  There are~$\binom{S}{S_0,\ldots,S_K}$ parent assignments consistent with the sequence sizes~$S_0,\ldots,S_K$, and they are all equally likely under the prior, so the conditional probability of the parent assignments is,
\begin{align}
    p(\{u_s\}_{s=1}^S \mid K, \Theta, \xi) &= \frac{1}{S_0!} (\lambda^\varnothing T)^{S_0} e^{-\lambda^\varnothing T} \prod_{k=1}^K \frac{\Gamma(S_k + \alpha)}{S_k! \, \Gamma(\alpha)} \left( \frac{\beta}{1+\beta} \right)^\alpha \left(\frac{1}{1+\beta} \right)^{N_k} \, \binom{S}{S_{0}, \ldots, S_{K}}^{-1} \\
    &= \frac{1}{S!} (\lambda^\varnothing T)^{S_0} e^{-\lambda^\varnothing T} \left( \frac{\beta}{1+\beta} \right)^{K\alpha} \left(\frac{1}{1+\beta} \right)^{S-S_0} \prod_{k=1}^K \frac{\Gamma(S_k + \alpha)}{\Gamma(\alpha)}.
\end{align}

The parent assignments above induce a partition, but technically they assume a particular ordering of the latent events.  Moreover, if some of the latent events fail to produce any observed events, they will not be included in the partition.  In performing a change of variables from parent assignments to partitions, we need to sum over latent event assignments that produce the same partition.  There are~${\binom{K}{K^*} K^*! = \frac{K!}{(K-K^*)!}}$ such assignments if there are~$K$ latent events but only~$K^*$ partitions. Thus,
\begin{align}
p(\mcI \mid K, \Theta, \xi)
&= \frac{K!}{(K-K^*)!} \, \frac{1}{S!} (\lambda^\varnothing T)^{|\mcI_0|} e^{-\lambda^\varnothing T} \left( \frac{\beta}{1+\beta} \right)^{K\alpha} \left(\frac{1}{1+\beta} \right)^{S-|\mcI_0|} \prod_{k=1}^{K^*} \frac{\Gamma(|\mcI_k| + \alpha)}{\Gamma(\alpha)}.
\end{align}
Clearly,~$K$ must be at least~$K^*$ in order to produce the partition. 

Finally, we sum over the number of latent events~$K$ to obtain the marginal probability of the partition,
\begin{align}
p(\mcI \mid \Theta, \xi) &= \sum_{K=K^*}^\infty \mathrm{Po}(K; \psi T) \, p(\mcI \mid K, \Theta, \xi) \\
&= V(K^*; \xi) \, \frac{(\lambda^\varnothing T)^{|\mcI_0|} e^{-\lambda^\varnothing T}}{S! \, (1+\beta)^{S-|\mcI_0|}} 
\prod_{k=1}^{K^*} \frac{\Gamma(|\mcI_k|+\alpha)}{\Gamma(\alpha)},
\end{align}
where 
\begin{align}
\label{eq:V}
V(K^*; \xi) &= \sum_{K=K^*}^\infty \mathrm{Po}(K; \psi T) \, \frac{K!}{(K-K^*)!} \left(\frac{\beta}{1+\beta}\right)^{K \alpha}.
\end{align}
\end{proof}

\subsection{Marginal Likelihood of Spikes in a Partition}
\label{sec:marginal_likelihood}
Now we combine the prior on partitions of spike indices with a likelihood of spikes under that partition. 
To match the notation in the main text and \Cref{sec:gibbs}, let~$X = \{X_k\}_{k=0}^{K^*}$ denote a partition of the spikes where~$X_k = \{x_s: s \in \mcI_k\}$ denote the set of spikes in sequence~$k$. 
In this section we derive the marginal probability of spikes in a sequence, integrating out the sequence time and type.  
For the background spikes, there are no parameters to integrate and the marginal probability is simply,
\begin{align}
    p(X_0 \mid \Theta) &= \prod_{x_s \in X_0} \mathrm{Unif}(t_s; [0, T]) \, \mathrm{Cat}(n_s \mid [\lambda_1^\varnothing / \lambda^\varnothing, \ldots, \lambda_N^\varnothing / \lambda^\varnothing]) \\
    &= T^{-|X_0|} \prod_{x_s \in X_0} \frac{\lambda_{n_s}^\varnothing}{\lambda^\varnothing}.
\end{align}
For the spikes attributed to a latent event, however, we have to marginalize over the type and time of the event.  We have,
\begin{align}
    p(X_k \mid \Theta) &= 
    \sum_{r_k=1}^R p(r_k) \int_0^T p(\tau_k) \prod_{s \in \mcI_k} p(t_s, n_s \mid r_k, \tau_k, \Theta) \, \mathrm{d} \tau_k\\
    &= \sum_{r_k=1}^R \pi_{r_k} \int_0^T \mathrm{Unif}(\tau_k; [0, T]) \prod_{x_s \in X_k} \mathrm{Cat}(n_s \mid r_k, \Theta) \, \mcN(t_s \mid \tau_k +b_{n_s r_k}, c_{n_s r_k}) \, \mathrm{d} \tau_k \\
    &= \frac{1}{T} \sum_{r_k=1}^R \pi_{r_k} \int_0^T \prod_{x_s \in X_k} \frac{a_{n_s r_k}}{\sqrt{2 \pi c_{n_s r_k}}} \exp \left \{ -\frac{1}{2 c_{n_s r_k}} (t_s - \tau_k - b_{n_s r_k})^2 \right \} \, \mathrm{d} \tau_k.
\end{align}
Let $J_{sk} = 1/c_{n_s r_k}$, and~$h_{sk} = (t_s - b_{n_s r_k})/c_{n_s r_k}$.  Then, 
\begin{align}
    p(X_k \mid \Theta)
    &= \frac{1}{T} \sum_{r_k=1}^R \pi_{r_k} \int_0^T \prod_{x_s \in X_k} a_{n_s r_k}   \sqrt{\frac{J_{sk}}{2 \pi}} \exp \left \{ -\tfrac{1}{2} J_{sk} \tau_k^2  + h_{sk} \tau_k  -\tfrac{1}{2} h_{sk}^2 J_{sk}^{-1} \right \} \, \mathrm{d} \tau_k \\
    &= \frac{1}{T} \sum_{r_k=1}^R \pi_{r_k}  \prod_{x_s \in X_k} \frac{a_{n_s r_k}}{ Z(J_{sk}, h_{sk})} \int_0^T  \exp \left \{ -\tfrac{1}{2} ({\textstyle \sum_{x_s \in X_k}} J_{sk}) \tau_k^2 + ({\textstyle \sum_{x_s \in X_k}} h_{sk}) \tau_k  \right \} \, \mathrm{d} \tau_k \\
    &\approx \frac{1}{T} \sum_{r_k=1}^R \pi_{r_k} \left( \prod_{x_s \in X_k} a_{n_s r_k} \right) \frac{Z({\textstyle \sum_{x_s \in X_k}} J_{sk}, {\textstyle \sum_{x_s \in X_k}} h_{sk})}{\prod_{x_s \in X_k}Z(J_{sk}, h_{sk})},
\end{align}
where
\begin{align}
    Z(J, h) = (2 \pi)^{1/2} J^{-1/2} \exp\left \{\tfrac{1}{2} h^2 J^{-1} \right\}
\end{align}
is the normalizing constant of a Gaussian density in information form with precision~$J$ and linear coefficient~$h$. For a sequence with a single spike, this reduces to,
\begin{align}
    p(x_s \mid \Theta)
    &\approx \frac{1}{T} \sum_{r_k=1}^R \pi_{r_k} a_{n_s r_k}.
\end{align}
The approximations above are due to the truncated integral over~$[0, T]$ rather than the whole real line.  In practice, this truncation is negligible for sequences far from the boundaries of the interval.

\subsection{Collapsed Gibbs Sampling the Partitions}
The conditional distribution on partitions given data and model parameters is given by,
\begin{align}
    p(\mcI \mid \{x_s\}_{s=1}^S, \Theta, \xi) &\propto 
    p(\mcI \mid \Theta, \xi) \, p(\{x_s\}_{s=1}^S \mid \mcI, \Theta)
    \label{eq:conditional-distribution-on-partition-1}
    \\
    &= p(\mcI \mid \Theta, \xi) \, p(X_0 \mid \Theta) \, \prod_{k=1}^{K^*} p(X_k \mid \Theta).
    \label{eq:conditional-distribution-on-partition-2}
\end{align}
Let~$X \setminus x_s$ denote the partition with spike~$x_s$ removed.  
Though it is slightly inconsistent with the proof in~\Cref{subsec:partitions}, let~$u_s$ denote the assignment of spike~$x_s$, as in the main text. 
Moreover, let~$S_k$ denote the size of the~$k$-th sequence in the partition, \emph{not including spike~$x_s$}.
We now consider the relative probability of the full partition~$X$ when~$x_s$ is added to each one of the possible parts: the background, an existing sequence, or a new sequence.

 For the background assignment,
\begin{align}
    p(u_s=0 \mid x_s, X \setminus x_s, \Theta, \xi) &\propto \frac{p(u_s=0, x_s, X \setminus x_s, \Theta, \xi)}{p(X \setminus x_s, \Theta, \xi)} \\
    &= \frac{V(K^*; \xi) \, \frac{(\lambda^\varnothing T)^{S_0+1} e^{-\lambda^\varnothing T}}{S! \, (1+\beta)^{S-S_0-1}} p(X_0 \cup x_s \mid \Theta) \prod_{k=1}^{K^*} \frac{\Gamma(S_k+\alpha)}{\Gamma(\alpha)} p(X_k \mid \Theta)}{V(K^*; \xi)  \, \frac{(\lambda^\varnothing T)^{S_0} e^{-\lambda^\varnothing T}}{S! \, (1+\beta)^{S-S_0}} p(X_0 \mid \Theta) \prod_{k=1}^{K^*} \frac{\Gamma(S_k+\alpha)}{\Gamma(\alpha)} p(X_k \mid \Theta)}\\
    &= \lambda^\varnothing T (1+\beta) \frac{1}{T} \frac{\lambda_{n_s}^\varnothing}{\lambda^\varnothing} \\
    &= (1+\beta) \lambda_{n_s}^\varnothing.
\end{align}
By a similar process, we arrive at the conditional probabilities of adding a spike to an existing sequence,
\begin{align}
    p(u_s=k \mid x_s, X \setminus x_s, \Theta, \xi) 
    &\propto \frac{\Gamma(S_k + \alpha + 1)}{\Gamma(S_k + \alpha)} \frac{p(X_k \cup x_s \mid \Theta)}{p(X_k \mid \Theta)} \\
    &= (S_k + \alpha) \, p(x_s \mid X_k, \Theta).
\end{align}
The predictive likelihood can be obtained via the ratio of marginal likelihoods derived above, or by explicitly calculating the categorical posterior distribution on types~$r_k$ and then the categorical and Gaussian posterior predictive distributions of~$n_s$ and $t_s$, respectively, given the type and the other spikes.  The latter approach is what is presented in the main text.

Finally, the conditional probability of assigning a spike to a new sequence simplifies as,
\begin{align}
    p(u_s=K^* + 1 \mid X \setminus x_s, \Theta, \xi) 
    &\propto  \left( \frac{\Gamma(\alpha + 1)}{\Gamma(\alpha)}\right) \left( \frac{V(K^* + 1; \xi)}{V(K^*; \xi)} \right) p(x_s \mid \Theta) \\
    &= \alpha  \left( \frac{\sum_{J=K^*+1}^\infty \frac{1}{J!} e^{-\psi T} (\psi T)^J \, \frac{J!}{(J-K^*-1)!} \left(\frac{\beta}{1+\beta}\right)^{J \alpha}}{\sum_{K=K^*}^\infty \frac{1}{K!} e^{-\psi T} (\psi T)^K \, \frac{K!}{(K-K^*)!} \left(\frac{\beta}{1+\beta}\right)^{K \alpha}} \right)   \left(\frac{1}{T} \sum_{r_k=1}^R \pi_{r_k} a_{n_s r_k}\right) \\
    &= \alpha \left( \frac{\sum_{K=K^*}^\infty (\psi T)^{K+1} \, \frac{1}{(K-K^*)!} \left(\frac{\beta}{1+\beta}\right)^{(K+1) \alpha}}{\sum_{K=K^*}^\infty (\psi T)^K \, \frac{1}{(K-K^*)!} \left(\frac{\beta}{1+\beta}\right)^{K \alpha}} \right)  \, \left(\frac{1}{T} \sum_{r_k=1}^R \pi_{r_k} a_{n_s r_k}\right)  \\
    &= \alpha \left(\frac{\beta}{1 + \beta} \right)^{\alpha}  \psi T  \left(\frac{1}{T} \sum_{r_k=1}^R \pi_{r_k} a_{n_s r_k}\right) \\
    &= \alpha \, \left(\frac{\beta}{1 + \beta} \right)^{\alpha} \psi \sum_{r_k=1}^R \pi_{r_k} a_{n_s r_k}.
\end{align}

\subsection{Gibbs Sampling the Latent Event Parameters}

Given the partition and global parameters, it is straightforward to sample the latent events.  In fact, most of the calculations were derived in \Cref{sec:marginal_likelihood}.  The conditional distribution of the event type is,
\begin{align}
    p(r_k \mid X_k, \Theta) &\propto
    p(r_k) \int_0^T p(\tau_k) \prod_{x_s \in X_k} p(t_s, n_s \mid r_k, \tau_k, \Theta) \, \mathrm{d} \tau_k\\
    &\approx \pi_{r_k} \left( \prod_{x_s \in X_k} a_{n_s r_k} \right) \frac{Z({\textstyle \sum_{x_s \in X_k}} J_{sk}, {\textstyle \sum_{x_s \in X_k}} h_{sk})}{\prod_{x_s \in X_k}Z(J_{sk}, h_{sk})},
\end{align}
where, again, the approximation comes from truncating the integral to the range~$[0, T]$, and is negligible in our cases.

Given the type, the latent event time is conditionally Gaussian,
\begin{align}
    p(\tau_k \mid r_k, X_k, \Theta) &\propto
    p(\tau_k) \prod_{x_s \in X_k} p(t_s \mid r_k, n_s, \tau_k, \Theta) \\
    &= \mathrm{Unif}(\tau_k; [0, T]) \prod_{x_s \in X_k} \mcN(t_s \mid \tau_k + b_{n_s r_k}, c_{n_s r_k}) \\
    &\propto \mcN(\tau_k \mid \mu_k, \sigma_k^2) 
\end{align}
where
\begin{align}
    \sigma_k^2 &= \left( \sum_{x_s \in X_k} J_{sk} \right)^{-1} \\
    \mu_k^2 &= \sigma_k^2 \left( \sum_{x_s \in X_k} h_{sk} \right).
\end{align}

Finally, the amplitude is independent of the type and time, and its conditional distribution is given by,
\begin{align}
    p(A_k \mid X_k, \xi) 
    &\propto \mathrm{Ga}(A_k \mid \alpha, \beta) \,
    \mathrm{Po}(S_k \mid A_k) \\
    &\propto \mathrm{Ga}(A_k \mid \alpha + S_k, \, \beta_\varnothing + 1).
\end{align}

\subsection{Gibbs Sampling the Global Parameters}
Finally, the Gibbs updates of the global parameters are simple due to their conjugate priors.

\paragraph{Background rates}
The conditional distribution of the background rates is,
\begin{align}
    p(\lambda_n^\varnothing \mid X_0, \xi) 
    &\propto \mathrm{Ga}(\lambda_n^\varnothing \mid \alpha_\varnothing, \beta_\varnothing) \,
    \mathrm{Po}(|\{x_s: x_s \in X_0, n_s = n\}| \mid \lambda_n^\varnothing T) \\
    &\propto \mathrm{Ga}(\lambda_n^\varnothing \mid \alpha_\varnothing + \sum_{x_s \in X_0} \mathbb{I}[n_s = n], \, 
    \beta_\varnothing + T).
\end{align}

\paragraph{Sequence type probabilities}
The conditional distribution of the sequence type probability vector~$\pi$ is,
\begin{align}
    p(\pi \mid \{(r_k, \tau_k, A_k)\}_{k=1}^{K^*}, \xi) 
    &\propto \mathrm{Dir}(\pi \mid \gamma \mathbf{1}_R) \prod_{k=1}^{K^*} \mathrm{Cat}(r_k \mid \pi) \\
    &\propto \mathrm{Dir}\left(\pi \mid \left[\gamma + \sum_k \mathbb{I}[r_k = 1], \ldots,  \gamma + \sum_k \mathbb{I}[r_k = 1] \right] \right).
\end{align}

\paragraph{Neuron weights for each sequence type}
The conditional distribution of the neuron weight vector~$\mathbf{a}_r$ is,
\begin{align}
    p(\mathbf{a}_r \mid X, \{(r_k, \tau_k, A_k)\}_{k=1}^{K^*}, \xi) 
    &\propto \mathrm{Dir}(\mathbf{a}_r \mid \varphi \mathbf{1}_N) \prod_{k=1}^{K^*} \prod_{x_s \in X_k} \mathrm{Cat}(n_s \mid \mathbf{a}_{r_k}), \\
    &\propto \mathrm{Dir}\left(\mathbf{a}_r \mid \boldsymbol{\phi}_r \right) \\
    \phi_{rn} &= \varphi + \sum_k \sum_{x_s \in X_k} \mathbb{I}[r_k = r \wedge n_s = n].
\end{align}

\paragraph{Neuron widths and delays}
The conditional distribution of the neuron delays~$b_{nr}$ and widths~$c_{nr}$ is,
\begin{align}
    p(b_{nr}, c_{nr} \mid X, \{(r_k, \tau_k, A_k)\}_{k=1}^{K^*}, \xi)
    &\propto \mathrm{Inv-}\chi^2(c_{nr}\mid \nu, \sigma^2) \, \mcN(b_{nr}\mid 0, c_{nr}/\kappa) \prod_{k=1}^{K^*} \prod_{x_s \in X_k} \left(\mcN(t_s \mid \tau_k + b_{nr}, c_{nr}) \right)^{\mathbb{I}[r_k=r \wedge n_s=n]} \\
    &\propto \mathrm{Inv-}\chi^2(c_{nr} \mid \nu_{nr}, \sigma_{nr}^2) \, \mcN(b_{nr} \mid \mu_{nr}, c_{nr} / \kappa_{nr}) 
\end{align}
where
\begin{align}
    S_{nr} &= \sum_{k=1}^{K^*} \sum_{x_s \in X_k} {\mathbb{I}[r_k=r \wedge n_s=n]} \\
    \nu_{nr} &= \nu + S_{nr} \\
    \sigma_{nr}^2 &= \frac{\nu \sigma^2 + \sum_{k=1}^{K^*} \sum_{x_s \in X_k} (t_s - \tau_k)^2 \cdot \mathbb{I}[r_k=r \wedge n_s=n]}{\nu + S_{nr}} \\
    \kappa_{nr} &= \kappa + S_{nr} \\
    \mu_{nr} &= \frac{\sum_{k=1}^{K^*} \sum_{x_s \in X_k} (t_s - \tau_k) \cdot \mathbb{I}[r_k=r \wedge n_s=n]}{\kappa + S_{nr}}.
\end{align}

\subsection{Split-Merge Sampling Moves}

As discussed in the main text, when sequences containing of a small number of spikes are unlikely under the prior, Gibbs sampling can be slow to mix over the posterior of spike partitions since the probability of forming new sequences (i.e. singleton clusters) is very low.
In addition to the annealed sampling approach outlined in the main text, we adapted the split-merge sampling method proposed by \textcite{Jain2004} for Dirichlet process mixture models to PP-Seq.\footnote{See also \textcite{Jain2007} for the case of non-conjugate priors.}
For simplicity, we implemented randomized split-merge proposals---i.e. without the additional restricted Gibbs sweeps described proposed in \textcite{Jain2004}.
In practice, we found that these random proposals were sufficient for the sampler to accept both split and merge moves at a satisfactorily high rate.

Briefly, our method starts by randomly choosing two spikes $(t_i, n_i)$ and $(t_j, n_j)$ that are not assigned to the background partition and satisfy $|t_i - t_j| < W$ for some user-specified time window $W$.
One could set $W = T$, in which case all pairs of spikes not assigned to the background are considerd; however, we will see that this will generally result in proposals that are extremely likely to be rejected, so it is advisable to set $W$ to be an upper bound on the expected sequence length to encourage faster mixing.
After identifying the pair of spikes, we propose to either merge their sequences (if $u_i \neq u_j$), or, if they belong to the same sequence ($u_i = u_j = k$) we propose to form two new sequences, each containing one of the two spikes, the remaining $|X_k| - 2$ spikes are randomly assigned with equal probability.

Given the proposed split or merge move, $\mcI \rightarrow \mcI^*$, the Metropolis-Hastings acceptance probability is:
\begin{equation}
\min \left [ 1, \frac{q(\mcI \mid \mcI^* )}{q(\mcI^* \mid \mcI )} \frac{p (\mcI^* \mid \{ x_s \}_{s=1}^S , \Theta, \xi) )}{p (\mcI \mid \{ x_s \}_{s=1}^S , \Theta, \xi)} \right ]
\end{equation}
where $q(\mcI^* \mid \mcI )$ is the proposal density, and $p (\mcI \mid \{ x_s \}_{s=1}^S , \Theta, \xi)$ is given in \cref{eq:conditional-distribution-on-partition-1,eq:conditional-distribution-on-partition-2}.
The sampling method then directly follows \textcite{Jain2004}.
The ratio of proposal probabilities for split moves is:
\begin{equation}
\frac{q(\mcI \mid \mcI^\text{split})}{q(\mcI^\text{split} \mid \mcI)} = \left ( \frac{1}{2} \right )^{-(|X_k| - 2)}
\end{equation}
And the ratio for merge moves is:
\begin{equation}
\frac{q(\mcI \mid \mcI^\text{merge})}{q(\mcI^\text{merge} \mid \mcI)} = \left ( \frac{1}{2} \right )^{(|X_k| - 2)}
\end{equation}

\section{Time-Warped Sequences}
To account for sequences that unfold at different speeds, we incorporate a time-warping component into the generative model.  Each sequence is endowed with a \emph{warp value}, which is drawn from a discrete distribution, 
\begin{align}
    \omega_k &\sim \sum_{f=1}^F \eta_f \delta_{w_f}(\omega_k),
\end{align}
where~$\eta_f \geq 0$ and~$\sum_{f=1}^F \eta_f = 1$ are the probabilities of the corresponding warp values~$w_f > 0$.  In this paper, we set
\begin{align}
    w_f &= w_F^{-1 + 2(f-1)/(F-1)} \\
    \eta_f &\propto \mcN\left(f \mid \tfrac{F-1}{2}, \sigma_w^2 \right).
\end{align}
The hyperparameters include the number of warp values~$F$, the maximum warp value~$w_F$, and the variance of the warp probabilities~$\sigma_w^2$.  We use an odd number of warp values so that~$w_{(F-1)/2} = 1$.

The warp value changes the distribution of spike times in the corresponding sequence by scaling the mean and standard deviation,
\begin{align}
t_s^{(k)} &\sim \mcN(\tau_k + \omega_k b_{n_s^{(k)} r_k}, \omega_k^2 c_{n_s^{(k)} r_k}).
\end{align}
Note that one could also choose to linearly scale the variance rather than the standard deviation, as in a warped random walk.  Ultimately, this is a subjective modeling choice, and scaling the standard deviation leads to slightly easier Gibbs updates later on.

We can equivalently endow each sequence with a latent \emph{warp index}~$f_k \in \{1, \ldots, F\}$ where~$p(f_k) \propto \eta_{f_k}$ and,
\begin{align}
t_s^{(k)} &\sim \mcN\left(\tau_k + w_{f_k} b_{n_s^{(k)} r_k}, w_{f_k}^2 c_{n_s^{(k)} r_k} \right).
\end{align}
Since there are only a finite number of warp values, we can treat the warp index and the sequence type as a single discrete latent variable~$(r_k, f_k) \in \{1, \ldots, R\} \times \{1, \ldots, F\}$.  This formulation allows us to re-use all of the collapsed Gibbs updates derived above, replacing our inferences over~$r_k$ with inferences of~$(r_k, f_k)$.  Computationally, this increases the run-time by a factor of~$F$.

Given the sequence types, times, and warp values, the conditional distribution of the neuron delays~$b_{nr}$ and widths~$c_{nr}$ is,
\begin{align}
    p(b_{nr}, c_{nr} \mid X, & \{(r_k, \tau_k, f_k, A_k)\}_{k=1}^{K^*}, \xi) \\
    &\propto \mathrm{Inv-}\chi^2(c_{nr}\mid \nu, \sigma^2) \, \mcN(b_{nr}\mid 0, c_{nr}/\kappa) \prod_{k=1}^{K^*} \prod_{x_s \in X_k} \left(\mcN(t_s \mid \tau_k + w_{f_k} b_{nr}, w_{f_k}^2 c_{nr}) \right)^{\mathbb{I}[r_k=r \wedge n_s=n]} \\
    &\propto \mathrm{Inv-}\chi^2(c_{nr} \mid \nu_{nr}, \sigma_{nr}^2) \, \mcN(b_{nr} \mid \mu_{nr}, c_{nr} / \kappa_{nr}) 
\end{align}
where
\begin{align}
    S_{nr} &= \sum_{k=1}^{K^*} \sum_{x_s \in X_k} {\mathbb{I}[r_k=r \wedge n_s=n]} \\
    \Delta_{sk} &= \frac{t_s - \tau_k}{w_{f_k}} \\
    \nu_{nr} &= \nu + S_{nr} \\
    \sigma_{nr}^2 &= \frac{\nu \sigma^2 + \sum_{k=1}^{K^*} \sum_{x_s \in X_k} \Delta_{sk}^2 \cdot \mathbb{I}[r_k=r \wedge n_s=n]}{\nu + S_{nr}} \\
    \kappa_{nr} &= \kappa + S_{nr} \\
    \mu_{nr} &= \frac{\sum_{k=1}^{K^*} \sum_{x_s \in X_k} \Delta_{sk} \cdot \mathbb{I}[r_k=r \wedge n_s=n]}{\kappa + S_{nr}}.
\end{align}
If instead you choose to linearly scale the variances, as mentioned above, the conditional distribution of the delays and widths is no longer inverse~$\chi^2$.  However, it is straightforward to sample the delays and widths one at a time from their univariate conditional distributions. 

\section{Supplemental Details on Experiments}

\subsection{ROC curve comparison of convNMF and PP-Seq}

Simulated datasets consisted of $N=100$ neurons, $T=2000$ time units, $R=1$ sequence type, sequence event rate $\psi = 0.02$, a default average background rate of $\lambda_n^\varnothing = 0.03$, default settings of $\alpha = 225$ and $\beta = 7.5$.
Further, we set $\kappa_{nr} = c_{nr}$ and $c_{nr} = 0.04$ so that neuron offsets had unit standard deviation and small jitter by default.
For the case of ``jitter noise'' we modified $c_{nr}$ but continued to set $\kappa_{nr} = c_{nr}$ to preserve the expected distribution of neuron offsets.

For each synthetic dataset, we fit convNMF models with $R=1$ component and a time bin size of 0.2 time units.
We optimized using alternating projected gradient descent with a backtracking line search.
Each convNMF model produced a temporal factor $\bh \in \R_+^B$ (where $B = T / 0.2 = 40000$ is the total number of time bins).
We encoded the ground truth sequence times in a binary vector of length $B$ and computed ROC curves by thresholing $\bh$ over a fine grid of values over the range $[0, \text{max}(\bh)]$.
We consider each timebin, indexed by $b$, a false positive if $h_b$ exceeds the threshold, but no ground truth sequence was present in the timebin.
Likewise, a timebin is a true positive when $h_b$ exceeds the threshold and a sequence event did occur in this timebin.

We repeated a similar analysis with 100 (approximate) samples from the PP-Seq posterior distribution.
Specifically, we discretized the sampled sequence event times, i.e. $\{\tau_k \}_{k=1}^{K}$ for every sample, into the same $B$ time bins used by convNMF.
For every timebin we computed the empirical probability that it contained a sampled event from the model, averaging over MCMC samples.
This resulted in a discretized, nonnegative temporal factor that is directly analogous to the factor $\bh$ produced convNMF.
We compute an ROC curve by the same method outlined above.

Unfortunately, both convNMF and PP-Seq parameters are only weakly identifiable.
For example, in PP-Seq, one can add a constant to all latent event times, $\tau_k \mapsto \tau_k + C$, and subtract the same constant from all neuron offsets $b_{nr} \mapsto b_{nr} - C$.
This manipulation does not modify the likelihood of the data (it does, however, affect the probability of the parameters under the prior, and so we say the model is ``weakly identifiable'').
The result of this is that both PP-Seq and convNMF may consistently misestimate the ``true'' sequence times by a small constant of time bins.
To discount this nuisance source of model error we repeated the above analysis on shifted copies of the temporal factor (up to 20 time bins in each direction); we then selected the optimal ROC curve for each model and computed the area under the curve to quantify accuracy.

\subsection{Hyperparameter sweep on hippocampal data}

We performed MCMC inference on 2000 randomly sampled hyperparameter sets, masking out a random 7.5\% of the data as a validation set on every run.
We fixed the following hyperparameter values: $\gamma = 3$, $\varphi = 1$, $\mu_{nr} = 0$ for every neuron and sequence type, $\nu_{nr} = 4$ for every neuron and sequence type.
We also fixed the following hyperparameters, pertaining to time warping: $F = 10$, $\sigma^2_w = 100$.
We randomized the number of sequence types $R \in \{1, 2, 3, 4\}$ uniformly.
We randomized the maximum warp value $w_F$ uniformly on (1, 1.5).
We randomized the mean sequence amplitude, $\alpha / \beta$, log-uniformly over the interval $(10^{2}, 10^{4})$, and set the variance of the sequence amplitude $\alpha / \beta^2$ equal to the mean.
We randomized the expected total background amplitude, $\lambda^\varnothing$, log-uniformly over the interval $(10^{2}, 10^{4})$ and set the variance of this parameter equal to its mean.
We randomized the sequence rate $\psi$ log-uniformly over $(10^{-3}, 10^{-1})$.
Finally, we randomized the neuron response widths $c_{nr}$ log-uniformly over $(10^{-1}, 10^{1})$.

We randomly initialized all spike assignments either with convNMF or the annealed MCMC procedure outlined in the main text (we observed successful outcomes in both cases). 
In annealed runs, we set an initial temperature of 500 and exponentially relaxed the temperature to one over 20 stages, each containing 100 Gibbs sweeps.
(Recall that this temperature parameter scales the variance of $\mathrm{Ga}(\alpha, \beta)$ while preserving the mean of this prior distribution.)
After annealing was completed we performed 100 final Gibbs sweeps over the data, the final half of which were used as approximate samples from the posterior.
In this final stage of sampling, we also performed 1000 randomized split-merge Metropolis Hastings moves after every Gibbs sweep.

All MCMC runs were performed using our Julia implementation of PP-Seq, run in parallel on the Sherlock cluster at Stanford University.
Each job was allocated 1 CPU and 8 GB of RAM.

\end{appendices}

\clearpage
\printbibliography

\end{changemargin}
\end{refsection}
\end{document}